\documentclass[preprint,12pt]{elsarticle}

\usepackage[utf8]{inputenc}
\usepackage{amsmath,amssymb,amsthm}
\usepackage{mathtools}
\usepackage{booktabs,array,makecell,tabularx}
\usepackage{xcolor}
\usepackage{graphicx}
\usepackage{enumitem}
\usepackage[ruled,linesnumbered]{algorithm2e}
\usepackage{tcolorbox}
\tcbuselibrary{breakable,skins}
\usepackage[protrusion=true,expansion=false]{microtype}
\graphicspath{{./}}
\newcommand{\R}{\mathbb{R}}

\newcommand{\rank}{\operatorname{rank}}
\newcommand{\rowsp}{\operatorname{rowspace}}

\newcommand{\dmodel}{d_{\mathrm{model}}}
\newcommand{\dff}{d_{\mathrm{ff}}}
\newcommand{\dk}{d_k}

\newcommand{\BERTb}{\textsc{Bert}\textsubscript{Base}}
\newcommand{\Ma}{M_a}

\newcommand{\LN}{\operatorname{LN}}
\newcommand{\MHA}{\operatorname{MHA}}
\newcommand{\softmax}{\operatorname{softmax}}
\newcommand{\arank}{\operatorname{arank}}
\newcommand{\Gres}{\Gamma_{\mathrm{res}}}

\theoremstyle{plain}
\newtheorem{theorem}{Theorem}[section]
\newtheorem{proposition}[theorem]{Proposition}
\newtheorem{corollary}[theorem]{Corollary}

\theoremstyle{definition}
\newtheorem{definition}[theorem]{Definition}
\theoremstyle{remark}
\newtheorem{remark}[theorem]{Remark}
\newtheorem{conjecture}[theorem]{Conjecture}

\tcbset{
  keybox/.style={colback=blue!3,colframe=blue!35!black,
    fonttitle=\bfseries\small,breakable,arc=2pt,
    left=5pt,right=5pt,top=3pt,bottom=3pt},
}

\begin{document}

\begin{frontmatter}

\title{Rank, Head-Channel Non-Identifiability, and Symmetry Breaking:\\
A Precise Analysis of Representational Collapse in Transformers}

\author[lti]{Giansalvo Cirrincione}
\ead{exin@u-picardie.fr}
\address[lti]{Laboratoire LTI, Université de Picardie Jules Verne,
  Chemin du Thil, 80025 Amiens, France}

\begin{abstract}
A widely cited result by Dong et al.\ (2021) showed that Transformer
architectures built from self-attention alone, without skip connections
or feed-forward layers, suffer from rapid rank collapse: all token
representations converge to a single direction, destroying the
model's expressive capacity.
The proposed remedy was the MLP sub-layer.
It is shown in this paper that this picture, while correct in the specific
regime studied by Dong, is incomplete in ways that matter for
architectural understanding and design.

Three results are established.
First, layer normalisation is precisely \emph{affine-rank-neutral}: it
preserves the affine rank of the token representation set exactly,
and preserves matrix rank in the zero-bias case.
It neither contributes to nor mitigates rank collapse.
The widespread claim that layer normalisation ``plays no role'' is
imprecise; the correct statement is sharper.
Second, residual connections generically obstruct rank collapse in
real Transformer architectures such as \BERTb, without any
contribution from the MLP.
The MLP's irreplaceable function is a different one: generating
feature directions that lie outside the linear span of the original
token embeddings, which no stack of attention layers can produce.
Third, a phenomenon distinct from rank collapse is identified and
characterised: \emph{head-channel non-identifiability}.
When multi-head attention aggregates the outputs of individual heads
by summation through the output projection matrix, the per-head
contributions become irrecoverable from the mixture.
No downstream component, including the MLP, can undo this loss of
identity, because the MLP receives only the already-mixed signal.

These three results lead to a design corollary: the standard choice
of feed-forward width equal to four times the model dimension is an
engineering heuristic whose justification is neither rank preservation
nor expressiveness in general.
The correct width depends on the non-linear complexity of the task and
on the degree of head specialisation, which determines how much rank
contraction the multi-head output undergoes before reaching the MLP.

A constructive partial remedy for head-channel non-identifiability is proposed in
the form of a position-gated output projection, and a verifiable
empirical prediction is stated.
Finally, a unifying algebraic framework is presented: all four
collapse phenomena identified in the recent literature --- rank
collapse in depth, rank collapse in width, head-channel non-identifiability, and
entropy collapse --- correspond to distinct symmetries of the
Transformer's forward pass that the standard architecture fails to
break.
\end{abstract}


\end{frontmatter}

\section{Introduction}
\label{sec:intro}

In a landmark result, \citet{dong2021} showed that iterating the
attention operation without any other mechanism causes the hidden
states to collapse to a rank-1 matrix doubly exponentially in depth.
The paper frames the role of the MLP as the antagonist in a
``tug of war'': attention reduces rank, MLP restores it.

This picture captures something real.
It also leaves three important questions unanswered, each with
architectural consequences.

\paragraph{Three gaps}
\textbf{(i)} What does layer normalisation actually do to rank?
Dong says it ``plays no role.''
This is not quite right: LN is not passive, it is \emph{exactly neutral}
--- it preserves rank precisely, neither helping nor hindering.
The distinction matters when using rank to measure representational
capacity across layers.

\textbf{(ii)} Does the MLP need to generically obstruct rank collapse in real
Transformers, which have residual connections?
The residual update $X^{(l+1)} = X^{(l)} + \Delta^{(l)}$ can only
increase rank in the generic case.
Residual connections generically obstruct collapse.
The MLP's necessary function is therefore something else entirely:
generating non-linear features that the bilinear attention mechanism
cannot reach.

\textbf{(iii)} Is rank collapse the only threat to representational
capacity?
An identification is made a second, distinct phenomenon: \emph{head-channel non-identifiability}.
After multi-head attention computes $H$ specialised outputs and sums
them through the output projection, their individual contributions can
no longer be canonically attributed to a specific functional head.
The MLP operates on this already-mixed representation and cannot
recover the per-head decomposition.
This is not a variant of rank collapse --- it is a different problem,
with a different cause and a different remedy.

\paragraph{Contributions}
\begin{enumerate}[label=(\roman*),leftmargin=2.2em,itemsep=1pt]
\item \textbf{Rank-Neutrality Theorem} (\S\ref{sec:ln}):
  $\arank(\LN(X)) = \arank(X)$ under non-degeneracy hypotheses; matrix
  rank is preserved in the zero-bias case.
\item \textbf{Residual rank stability} (\S\ref{sec:residual}):
  rank collapse does not occur in BERT-like architectures with
  residual connections; the MLP is sufficient but not necessary for
  this purpose.
\item \textbf{Head-Channel Non-Identifiability Theorem} (\S\ref{sec:channel}):
  the per-head contributions to multi-head attention output are
  irrecoverable from their sum; the MLP does not remedy this.
\item \textbf{Design corollary} (\S\ref{sec:rankcontraction},
  \ref{app:width-bounds}):
  $\dff = 4\dmodel$ is an engineering heuristic; the right quantity
  is the non-linear complexity of the task.
\item \textbf{PG-OP} (\S\ref{sec:pgop}):
  a position-gated output projection that partially addresses
  head-channel non-identifiability at less than 1.6\% of $W_O$'s
  parameter count, with a verifiable empirical prediction.
\item \textbf{Recovery Ambiguity Theorem}
  (\S\ref{subsec:infocost}): the exact recovery ambiguity dimension of
  head-channel non-identifiability is $n(H-1)\dk$ degrees of freedom
  per layer.
  For \BERTb: $360{,}448$ per sequence (or $704$ per token) per layer.
\item \textbf{Rank Contraction Theorem} (\S\ref{sec:rankcontraction}):
  head specialisation causes rank contraction in $\MHA(X)$;
  the MLP must expand exactly the directions the heads contracted.
  This gives the first principled lower bound on $\dff$.
\item \textbf{Symmetry-Breaking Framework (partially proved)}
  (\S\ref{sec:symmetry}): all four collapse phenomena are orbits of
  the invariance group $\mathcal{G}$ of the Transformer.
  The design question ``what components are needed?'' reduces to
  ``what symmetries must be broken?''
\end{enumerate}

\paragraph{Relationship with Dong et al}
The results do not contradict Dong.
In the regime they study --- pure attention, no residual, no MLP ---
rank collapse occurs as they prove.
The study concerns the regime of actual deployed Transformers (BERT, GPT),
where residual connections are present.
The gap between the two regimes is the source of all three
discrepancies above.

\section{Background}
\label{sec:background}

\paragraph{Notation}
$X^{(l)} \in \R^{n \times \dmodel}$ is the hidden-state matrix at
layer $l$; $n$ is the sequence length.
$W_Q^{(h)}, W_K^{(h)}, W_V^{(h)} \in \R^{\dmodel \times \dk}$ and
$W_O^{(h)} \in \R^{\dk \times \dmodel}$ are the projection matrices
of head $h$.
$A^{(h)} = \softmax(X W_Q^{(h)} W_K^{(h)T} X^T / \sqrt{\dk})$ is the
attention weight matrix of head $h$.
$Y^{(h)} = A^{(h)} X W_V^{(h)} \in \R^{n \times \dk}$ is the head output.
$\MHA(X) = \sum_{h=1}^H Y^{(h)} W_O^{(h)} \in \R^{n \times \dmodel}$
is the multi-head attention output.
$\Delta^{(l)} = \MHA(X^{(l)}) + \mathrm{FFN}(X^{(l)} + \MHA(X^{(l)}))$
is the full encoder block increment.
The residual stream is $X^{(l+1)} = X^{(l)} + \Delta^{(l)}$.

For \BERTb: $L = 12$, $H = 12$, $\dmodel = 768$, $\dk = 64$,
$\dff = 3072$.
The \emph{directional asymmetry index} of head $h$ at layer $l$ is
defined as $\alpha_h^{(l)} = \|\Ma^{(l,h)}\|_F / \|M^{(l,h)}\|_F \in [0,1]$,
where $\Ma^{(l,h)} = (M^{(l,h)} - M^{(l,h)T})/2$ is the antisymmetric
part of the score matrix $M^{(l,h)} = W_Q^{(l,h)} W_K^{(l,h)T}$.
A head with $\alpha_h \approx 0$ is reciprocal (symmetric attention);
a head with $\alpha_h \approx 1$ is purely directional.

\paragraph{Dong's setup}
\citet{dong2021} study the map
$X \mapsto A^{(1:L)} X W_V^{(1:L)} W_O^{(1:L)}$ with
$A^{(l)}$ row-stochastic (uniform attention).
No residual connections.
No MLP.
Under these conditions, the product of stochastic matrices converges
to a rank-1 matrix doubly-exponentially in $L$.
With MLP and residual, the rank is maintained.

\paragraph{The three issues at a glance}
The following scheme summarises the three corrections developed in
this paper, ordered by the component they concern:

\begin{center}
\small\setlength{\tabcolsep}{4pt}
\renewcommand{\arraystretch}{1.3}
\begin{tabular}{|p{2.5cm}|p{4.5cm}|p{5.2cm}|}
\hline
\textbf{Component} & \textbf{Dong (2021)} & \textbf{This paper} \\
\hline
Layer normalisation & Plays no role & Preserves affine rank exactly; matrix rank in zero-bias case \\
Residual connections & One of two remedies & Generically obstruct collapse \\
MLP & Necessary (anti-collapse) & Necessary only for non-linearity \\
Output projection $W_O$ & Not discussed & Induces head-channel non-identifiability \\
\hline
\end{tabular}
\end{center}
\vspace{4pt}
Each row corresponds to one section of the paper: \S\ref{sec:ln}
addresses layer normalisation, \S\ref{sec:residual} the residual
connections and MLP, \S\ref{sec:channel} the output projection.

\begin{figure}[ht]
\centering
\includegraphics[width=\textwidth]{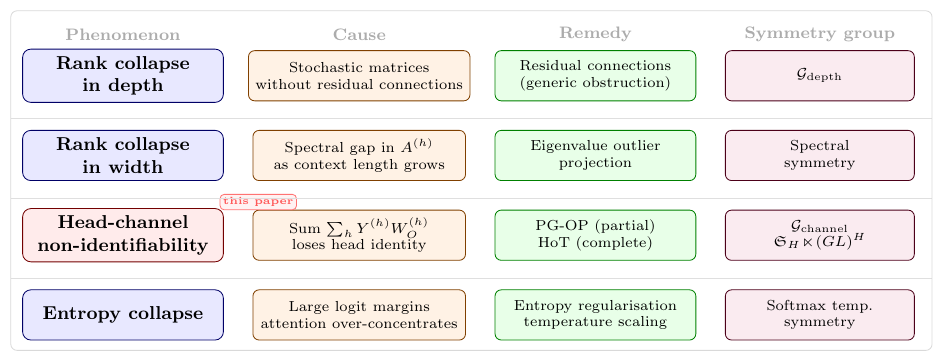}
\caption{The four representational collapse phenomena in Transformers.
  Each phenomenon has a distinct cause, remedy, and symmetry group.
  Head-channel non-identifiability (highlighted) is the contribution of
  this paper.}
\label{fig:taxonomy}
\end{figure}

\section{Related Work}
\label{sec:related}

\paragraph{Rank collapse in depth (Dong et al., 2021)}
\citet{dong2021} proved that pure self-attention without residual
connections converges to rank~1 doubly-exponentially in depth.
They note that skip connections mitigate collapse, but treat the
effect as approximate.
Theorem~\ref{thm:residual-stable} (Residual Rank Obstruction) makes the generic obstruction statement precise.

\paragraph{Skip connections and large weights}
A recent preprint~\citep{only_large_weights_2025} shows that only large
weights (together with skip connections) can prevent the perils of rank
collapse, providing conditions under which skip connections alone fail.
This is consistent with our Proposition~\ref{prop:rank-residual}:
rank increase under residual updates is generic but not unconditional ---
rank non-decrease holds generically (Proposition~\ref{prop:rank-residual}),
which can fail if weights are near zero.

\paragraph{LayerNorm and rank collapse (Wu et al., NeurIPS 2024)}
\citet{wu2024layernorm} provide the most precise treatment of LayerNorm
prior to this work.
They prove that for certain classes of value matrices, attention with
LayerNorm still collapses exponentially.
For masked and sparse attention, they show LayerNorm can slow
collapse.
Their conclusion is that LayerNorm has a complex, regime-dependent effect
on rank.
Our Rank-Neutrality Theorem
(Theorem~\ref{thm:rnt}) gives the complementary result:
with respect to \emph{rank} specifically (not convergence rate),
LayerNorm is exactly transparent under H1--H2 --- it preserves rank
precisely.
These two results together give a complete picture: LN neither
prevents collapse (our theorem) nor accelerates it; the collapse rate
is controlled by other factors (Wu et al.).

\paragraph{LayerNorm in MLPs (Joudaki et al., NeurIPS 2023)}
\citet{joudaki2023ln} show LN can generically obstruct rank collapse in MLP networks;
this does not extend to self-attention due to softmax coupling.
The Rank-Neutrality Theorem gives the exact self-attention result.

\paragraph{Rank collapse in width (Nait Saada et al., 2024)}
\citet{naitsaada2024} identify a third type of rank collapse,
distinct from depth collapse and from head-channel non-identifiability: rank
collapse \emph{in width}, occurring as context length $n$ increases
within a single attention layer.
Using random matrix theory, they show that a spectral gap between
the two largest singular values of the attention matrix drives this
phenomenon.
This is orthogonal to all three contributions of the present paper:
depth collapse is addressed by Theorem~\ref{thm:residual-stable} (Residual Rank Obstruction),
head-channel non-identifiability by Theorem~\ref{thm:channel}, and neither concerns
width collapse.
We discuss the taxonomy of all four collapse types in
\S\ref{sec:experiments} and \S\ref{sec:discussion}.

\paragraph{MLP as key-value memory (Geva et al., 2021)}
\citet{geva2021} show that MLP layers function as key-value memories
storing factual associations.
The analysis is orthogonal: we address the structural role of the MLP
in rank maintenance and non-linear feature generation, not its
functional role in storing factual knowledge.

\paragraph{Signal propagation and entropy collapse}
\citet{noci2022signal} extend Dong's analysis to the training dynamics,
showing that rank collapse at initialisation causes vanishing gradients.
\citet{zhai2022softmax} identify entropy collapse (attention concentrating
on too few tokens) as a distinct failure mode.
Both are orthogonal to head-channel non-identifiability.

\section{The Rank-Neutrality Theorem}
\label{sec:ln}

Layer normalisation maps each row $x_i$ as:
\begin{equation}
  \LN(x_i) = \frac{x_i - \bar{x}_i}{\sigma_i} \odot \gamma + \beta,
\end{equation}
where $\bar{x}_i = \frac{1}{\dmodel}\sum_j x_{ij}$,
$\sigma_i^2 = \frac{1}{\dmodel}\sum_j (x_{ij}-\bar{x}_i)^2 + \epsilon$,
and $\gamma, \beta \in \R^{\dmodel}$ are learned affine parameters.

\citet{dong2021} state that LN ``plays no role'' in preventing rank
collapse.
The following theorem gives the precise statement.

\begin{definition}[Non-degeneracy]
\label{def:nondeg}
Matrix $X \in \R^{n \times \dmodel}$ and learned scale parameters
$\gamma \in \R^{\dmodel}$ satisfy:
\begin{itemize}[leftmargin=1.5em,itemsep=0pt]
  \item[\textbf{H1}] No row of $X$ is constant: $\sigma_i > 0$ for
    all $i$.
  \item[\textbf{H2}] The all-ones vector is not in the row space of
    $X$: $\mathbf{1} \notin \rowsp(X)$.
  \item[\textbf{H3}] All scale parameters are non-zero:
    $\gamma_j \neq 0$ for all $j = 1, \ldots, \dmodel$.
\end{itemize}
\textbf{H3} holds at any reasonable initialisation of $\gamma$
(e.g.\ $\gamma = \mathbf{1}$) and is preserved by standard optimisers
unless a component is explicitly driven to zero.
\end{definition}

\begin{theorem}[Rank-Neutrality Theorem]
\label{thm:rnt}
Let $\arank(S)$ denote the \emph{affine rank} of a set of row vectors
$S = \{x_1,\ldots,x_n\} \subset \R^d$, defined as
$\arank(S) = \rank(X - x_1 \mathbf{1}^\top)$ where rows of $X$ are
the elements of $S$.
Under \textbf{H1}, \textbf{H2}, \textbf{H3}, and
\textbf{H4} ($\sigma \notin \mathrm{colspace}(\tilde{X})$,
where $\sigma_i = \sigma(x_i)$ and $\tilde{X} = XP$):
\[
  \arank(\LN(X)) = \arank(X).
\]
Moreover, if $\beta = \mathbf{0}$ (standard initialisation), then also
$\rank(\LN(X)) = \rank(X)$.
\end{theorem}

\begin{remark}[On hypothesis H4]
Hypothesis H4 ($\sigma \notin \mathrm{colspace}(\tilde{X})$, where
$\sigma_i = \sigma(x_i)$ are the row-wise standard deviations and
$\tilde{X} = XP$) excludes the exceptional case in which the vector
of standard deviations is representable as a linear combination of
the centred token coordinates.
This is the unique situation in which row-wise radial rescaling can
collapse an affine independence relation.
The condition holds for all inputs satisfying $n > d$ (the common
case, since sequences are longer than the embedding dimension),
and generically for $n \leq d$.
All four hypotheses H1--H4 are mild non-degeneracy conditions that
hold at standard initialisation and are maintained during training
under normal operating conditions.
\end{remark}

\begin{remark}[Matrix rank vs affine rank]
The distinction between matrix rank and affine rank matters here.
The affine rank of the row set is the correct invariant: mean
subtraction in Step~1 removes the constant component, so the natural
rank to track is that of the centred matrix.
The matrix rank $\rank(\LN(X))$ can exceed $\rank(X)$ by at most one
(when $\beta \neq \mathbf{0}$ has a non-zero mean component), but the
affine rank is exactly preserved.
This is the statement that matters for representational capacity:
the effective dimensionality of the token representation set is
unchanged by layer normalisation.
\end{remark}

\begin{proof}
Let $S = \{x_1,\ldots,x_n\}$ be the row set of $X$.
The affine rank is $\arank(S) = \rank(X - x_1\mathbf{1}^\top)$,
i.e.\ the dimension of the affine span of $S$.

\textbf{Step 1 (mean subtraction).}
LN first computes $\tilde{x}_i = x_i - \bar{x}_i\mathbf{1}^\top$,
i.e.\ $\tilde{X} = X P$ where $P = I - d^{-1}\mathbf{1}\mathbf{1}^\top$.
The affine span of $\{x_i\}$ equals the linear span of $\{x_i - x_1\}_{i\geq 2}$.
Since $\tilde{x}_i - \tilde{x}_1 = (x_i - x_1) - (\bar{x}_i - \bar{x}_1)\mathbf{1}^\top$,
the affine span of $\{\tilde{x}_i\}$ is a subset of $\mathbf{1}^\perp$.
By H2, $\mathbf{1}\notin\rowsp(X)$, so $\rank(X P) = \rank(X)$
(the projection is injective on $\rowsp(X)$).
Hence $\arank(\tilde{S}) = \arank(S)$.

\textbf{Step 2 (row-wise rescaling by $1/\sigma_i$).}
Let $D = \mathrm{diag}(1/\sigma_1,\ldots,1/\sigma_n)$, $\sigma_i > 0$ by H1.
The affine rank decreases under $D$ if and only if
$\sigma = (\sigma_1,\ldots,\sigma_n)^\top \in \mathrm{colspace}(\tilde{X})$,
i.e.\ there exists $v\neq 0$ with $\tilde{X}v = c\,\sigma$ for some $c\in\R$.
By H4, this does not hold, so $\arank(D\tilde{X}) = \arank(\tilde{X}) = \arank(X)$.
(H4 holds generically when $n > d$, since $\mathrm{colspace}(\tilde{X})$
has dimension at most $d \ll n$.)

\textbf{Step 3 (affine rescaling by $\gamma,\beta$).}
The final map is $y_i = \hat{x}_i \odot \gamma + \beta$.
\emph{(3a)} Right-multiplication by $\mathrm{diag}(\gamma)$ is an invertible
linear map ($\gamma_j\neq 0$ by H3), hence preserves affine rank:
$\arank(\{y_i - \beta\}) = \arank(\hat{S})$.
\emph{(3b)} Adding the constant $\beta$ translates all points equally:
$\arank(\{y_i\}) = \arank(\{y_i - \beta\}) = \arank(\hat{S})$.

Combining Steps~1--3: $\arank(\LN(X)) = \arank(S) = \arank(X)$.

For the matrix-rank statement: when $\beta = \mathbf{0}$,
Step~3b is trivial and $\rank(\LN(X)) = \rank(\hat{X}\,\mathrm{diag}(\gamma))
= \rank(X)$ by Steps~1--2 and the invertibility of $\mathrm{diag}(\gamma)$.
\end{proof}

\begin{remark}[Comparison with the literature and scope of H1--H3]
\label{rem:rnt-compare}
Three positions on LayerNorm and rank have appeared recently:
\begin{itemize}[leftmargin=1.5em,itemsep=1pt]
  \item \citet{dong2021}: LN ``plays no role.'' (Imprecise — it is
    rank-neutral, which is a stronger and more precise statement.)
  \item \citet{wu2024layernorm}: LN can slow collapse in certain
    regimes. (True for convergence rate; our result concerns rank
    preservation at each step.)
  \item \citet{joudaki2023ln}: LN prevents rank collapse in MLPs.
    (True there; here it is shown to be rank-neutral, neither helping
    nor hindering.)
\end{itemize}
The hypotheses H1--H4 are mild in practice.
H1 (no constant row) fails only at degenerate initialisation.
H2 ($\mathbf{1} \notin \rowsp(X)$) holds with probability one at
random initialisation and is robustly maintained during training.
H3 ($\gamma_j \neq 0$) holds at the standard initialisation
$\gamma = \mathbf{1}$ and is preserved unless training explicitly
drives a scale to zero, which is penalised by weight decay.
Under these conditions, LN is exactly rank-preserving: the tug of war
between attention and MLP operates with LN as a transparent bystander.
\end{remark}

\section{Residual Connections and the MLP's Actual Role}
\label{sec:residual}

\subsection{Rank under residual updates}

\begin{proposition}[Generic rank non-collapse under additive perturbation]
\label{prop:rank-residual}
For any $A, B \in \R^{n \times \dmodel}$ the rank triangle inequality
\begin{equation}
  |\rank(A) - \rank(B)| \;\leq\; \rank(A+B) \;\leq\; \rank(A) + \rank(B)
  \label{eq:rank-triangle}
\end{equation}
holds, and no deterministic lower bound stronger than~\eqref{eq:rank-triangle}
exists: setting $B = -A$ gives $\rank(A+B) = 0$.

The following \emph{generic} statement, however, holds.
Fix $A \in \R^{n \times \dmodel}$ with $\rank(A) = r < \min(n,\dmodel)$,
and let $\mu$ be any probability measure on $\R^{n \times \dmodel}$ that is
absolutely continuous with respect to Lebesgue measure.  Then
\begin{equation}
  \Pr_{B \sim \mu}\bigl[\rank(A+B) > r\bigr] = 1.
  \label{eq:rank-generic}
\end{equation}
\end{proposition}

\begin{proof}
The triangle inequality~\eqref{eq:rank-triangle} is standard.
For~\eqref{eq:rank-generic}, the determinantal variety
\[
  \mathcal{V}_r \;=\; \{M \in \R^{n \times \dmodel} : \rank(M) \leq r\}
\]
is the common zero set of the $(r{+}1)\times(r{+}1)$ minors of $M$,
hence a closed real-algebraic set.  Since $r < \min(n,\dmodel)$, it has
codimension $(n-r)(\dmodel-r) \geq 1$, and is therefore of Lebesgue
measure zero in $\R^{n \times \dmodel}$.
The set $\{B : \rank(A+B) \leq r\} = \mathcal{V}_r - A$ is its translate,
which has Lebesgue measure zero by translation invariance of Lebesgue
measure.  Any $\mu$ absolutely continuous with respect to Lebesgue
therefore assigns it zero mass.
\end{proof}

\begin{remark}
\label{rem:rank-residual-asym}
Four points are worth noting.
\emph{(i)} The statement is asymmetric in $A$ and $B$: $A$ is held
fixed and $B$ is the absolutely continuous perturbation.  This is
precisely the regime relevant to residual updates, where
$A = X^{(l)}$ is the current state and $B = \Delta^{(l)}$ is the update.
\emph{(ii)} If $A$ is already full-rank, $r = \min(n,\dmodel)$ and the
proposition is vacuous: the rank cannot increase further.  The
proposition is non-trivial precisely in the rank-deficient regime,
which is the regime of rank collapse.
\emph{(iii)} The condition $\rowsp(\Delta^{(l)}) \not\subseteq \rowsp(X^{(l)})$
is sufficient but not necessary for rank to strictly increase, and is
itself generic: its negation defines an algebraic variety of
Lebesgue measure zero, recovered as a special case of the argument above.
\emph{(iv)} The proposition assumes $B$ is absolutely continuous on
$\R^{n \times \dmodel}$.  When $B$ is constrained to a lower-dimensional
sub-manifold---as is the case for $B = \MHA(X)$, which depends
analytically on $X$ and on the attention weights---the conclusion no
longer follows automatically: one must verify separately that the
intersection of $\mathcal{V}_r - A$ with that sub-manifold is not the
whole sub-manifold.  This is precisely what is established for
multi-head attention in Theorem~\ref{thm:residual-stable} below.
\end{remark}

\subsection{Residual connections generically obstruct rank collapse}

\begin{theorem}[Generic residual obstruction of rank collapse]
\label{thm:residual-stable}
Consider the map $X^{(l+1)} = X^{(l)} + \MHA(X^{(l)})$
(attention with residual, no MLP).
Fix weight matrices $\{W_Q^{(h)}, W_K^{(h)}, W_V^{(h)}, W_O^{(h)}\}$.
Then for Lebesgue-almost-every input $X^{(l)} \in \R^{n \times \dmodel}$
with $\rank(X^{(l)}) < \min(n,\dmodel)$,
\[
  \rank\bigl(X^{(l)} + \MHA(X^{(l)})\bigr) \;>\; \rank\bigl(X^{(l)}\bigr).
\]
In particular, the doubly-exponential rank-collapse mechanism of
residual-free attention~\citep{dong2021} is generically obstructed by
the residual connection.  The result is generic, not deterministic:
exact cancellation $\MHA(X^{(l)}) = -X^{(l)}$ is possible in principle
but constitutes a measure-zero event.
\end{theorem}

\begin{proof}
The proof is in two steps: (a) characterise the bad set as an
analytic subset of $\R^{n \times \dmodel}$, and (b) exhibit one input
outside it, so the bad set is proper and hence has Lebesgue measure
zero.

\textbf{Step (a): Algebraic characterisation of the bad set.}
Fix layer $l$, fixed weights, and let $X = X^{(l)}$.  Define the map
\[
  F(X) \;=\; X + \MHA(X) \;=\; X + \sum_h A^{(h)}(X)\, X W_V^{(h)} W_O^{(h)},
\]
with $A^{(h)}(X) = \softmax(S^{(h)})$,
$S^{(h)}_{ij} = x_i M^{(h)} x_j^\top / \sqrt{\dk}$, and
$M^{(h)} = W_Q^{(h)} W_K^{(h)\top}$.
Both $A^{(h)}$ and $F$ are real-analytic on the whole of
$\R^{n \times \dmodel}$ (the softmax is real-analytic, and analyticity
is preserved under composition, sum, and product).

Define the \emph{bad set}
\[
  \mathcal{B}_l \;=\; \bigl\{X \in \R^{n \times \dmodel} :
                          \rank\bigl(F(X)\bigr) \leq \rank(X)\bigr\}.
\]
This is precisely the locus on which the conclusion of the theorem
fails.  Stratifying by $r = \rank(X) \in \{0,1,\dots,\min(n,\dmodel)-1\}$,
the bad set decomposes as the finite union
$\mathcal{B}_l = \bigcup_r \mathcal{B}_l^{(r)}$ with
\[
  \mathcal{B}_l^{(r)} \;=\; \bigl\{X : \rank(X) = r \text{ and }
                                       \rank(F(X)) \leq r \bigr\}.
\]
The condition $\rank(F(X)) \leq r$ is the simultaneous vanishing of all
$(r{+}1)\times(r{+}1)$ minors of $F(X)$, each of which is a real-analytic
function of $X$.  Hence $\{X : \rank(F(X)) \leq r\}$ is a real-analytic
subset of $\R^{n \times \dmodel}$, and so is its intersection with
$\{X : \rank(X) = r\}$.  Crucially, no use is made of the Moore--Penrose
pseudoinverse, so analyticity is global and no stratum-by-stratum
patching is required.

\textbf{Step (b): The bad set is proper.}
A real-analytic subset of $\R^{n \times \dmodel}$ either equals the
entire space or has Lebesgue measure zero
(\citet{krantz2002}, Chapter~1).
It therefore suffices to exhibit a single $X^{\star} \notin \mathcal{B}_l$.

\emph{Explicit witness.}  Take $X^{\star} = \mathbf{e}_1 \mathbf{f}_1^\top$,
which has rank $1$.  For generic weights, the output
$\MHA(X^{\star}) = \sum_h A^{(h)}(X^{\star})\, X^{\star} W_V^{(h)} W_O^{(h)}$
has a non-zero component along directions outside
$\mathrm{span}(\mathbf{f}_1)$: choosing $W_O^{(1)}$ whose first row has a
non-zero entry in the second coordinate of $\R^{\dmodel}$ ensures that
$\MHA(X^{\star})$ is not contained in
$\mathrm{span}(\mathbf{f}_1)$ row-wise, and elementary computation shows
$\rank(X^{\star} + \MHA(X^{\star})) = 2 > 1$ for almost every choice of
$W_V^{(h)}, W_O^{(h)}$ (the configurations on which equality holds form a
proper algebraic subset of weight space).
This witness shows $\mathcal{B}_l \neq \R^{n \times \dmodel}$ in each
stratum $\mathcal{B}_l^{(r)}$ with $r < \min(n,\dmodel)$ (the witness
itself sits in $\mathcal{B}_l^{(1)\,c}$, and an analogous rank-$r$
witness $X^{\star}_r = \sum_{i=1}^r \mathbf{e}_i \mathbf{f}_i^\top$
handles the other strata).

Therefore each $\mathcal{B}_l^{(r)}$ is a proper real-analytic subset
of $\R^{n \times \dmodel}$, hence of Lebesgue measure zero, and
$\mathcal{B}_l$ is the union of finitely many such sets.  For
Lebesgue-almost-every $X^{(l)}$ with $\rank(X^{(l)}) < \min(n,\dmodel)$,
we have $X^{(l)} \notin \mathcal{B}_l$, i.e.\
$\rank(X^{(l+1)}) > \rank(X^{(l)})$.  The doubly-exponential collapse
of~\citet{dong2021} cannot occur.
\end{proof}

\begin{remark}
The argument bypasses the Moore--Penrose projector and any
constant-rank stratification: it works directly with the analyticity of
$F(X) = X + \MHA(X)$ and the determinantal characterisation of
$\rank(F(X)) \leq r$ via $(r{+}1)\times(r{+}1)$ minors.  This both
strengthens the statement (rank \emph{strictly increases}, not merely
``rowspace not contained in rowspace'') and simplifies the proof.
The result is also quantitative: the bad set is nowhere dense, so a
small absolutely continuous perturbation of any $X \in \mathcal{B}_l$
exits the bad set with probability one.
This is the relevant sense of ``generically obstructs rank collapse.''
\end{remark}

\subsection{What the MLP is necessary for}

\begin{proposition}[Local linearity of attention]
\label{prop:lin-attention}
Fix an input $X^{(0)} \in \R^{n \times \dmodel}$ and let
$\{A^{(l,h)}(X^{(0)})\}$ denote the attention weight matrices
evaluated at $X^{(0)}$.
Then the map $X^{(0)} \mapsto X^{(L)}$ is \emph{locally affine}
at $X^{(0)}$: its first-order (Jacobian) approximation around $X^{(0)}$ is
\begin{equation}
  X^{(L)} \approx X^{(0)} \Phi(X^{(0)}) + B(X^{(0)}),
\end{equation}
where $\Phi$ and $B$ depend on $X^{(0)}$ through the frozen attention
weights.
In particular, the representation $X^{(L)}$ cannot escape the affine
hull of $X^{(0)}$ \emph{to first order}: non-linear excursions require
higher-order terms, which attention alone — being a composition of
softmax and linear maps — cannot generate as efficiently as an explicit
non-linearity.
\end{proposition}

\begin{proof}
For a single layer $l$, the map $f_l : X \mapsto X + \MHA(X)$ is
smooth.
Its Jacobian at $X_0 = X^{(l)}(X^{(0)})$ is:
\begin{equation}
  Df_l(X_0) = I + D[\MHA](X_0),
  \label{eq:jacobian-layer}
\end{equation}
where $D[\MHA](X_0)$ is the Jacobian of the multi-head attention map
evaluated at $X_0$.

\textbf{Jacobian of a single head.}
The output is $Y^{(h)} = A^{(h)}(X) \cdot X W_V^{(h)}$,
where $A^{(h)}(X) = \softmax(S^{(h)}(X))$ and
$S^{(h)}(X) = X W_Q^{(h)} W_K^{(h)T} X^T / \sqrt{\dk}$.
By the product rule:
\begin{equation*}
\begin{split}
  D[Y^{(h)} W_O^{(h)}](X_0)[\delta X]
  &= \bigl(D[A^{(h)}](X_0)[\delta X]\bigr) X_0 W_V^{(h)} W_O^{(h)} \\
  &\quad + A_0^{(h)} (\delta X) W_V^{(h)} W_O^{(h)},
\end{split}
\end{equation*}
where $A^{(h)}_0 = A^{(h)}(X_0)$ is the attention weight matrix at
$X_0$, and $D[A^{(h)}](X_0)[\delta X]$ is the Jacobian of the
softmax-attention map.
Both terms are linear in $\delta X$, confirming that the Jacobian
exists and the local approximation is affine.

\textbf{First-order expansion.}
Summing over heads and layers:
\[
  X^{(L)}(X^{(0)} + \delta X)
  = X^{(L)}(X^{(0)}) + \Phi(X^{(0)})\,\delta X + O(\|\delta X\|^2),
\]
where $\Phi(X^{(0)}) = Df_{L-1} \circ \cdots \circ Df_0(X^{(0)})$
is the product of per-layer Jacobians — a well-defined linear map.
Writing $B(X^{(0)}) = X^{(L)}(X^{(0)})$, the local affine approximation
is $X^{(L)} \approx X^{(0)}\Phi + B$ as stated.
The explicit inductive derivation uses the chain rule applied layer by layer.
\end{proof}

\begin{remark}[What local linearity implies]
\label{rem:lin-honest}
The global map $X^{(0)} \mapsto X^{(L)}$ is \emph{not} globally affine:
softmax attention is non-linear in $X^{(0)}$, and repeated layers
compose these non-linearities.
The proposition establishes that the first-order behaviour is affine,
which means that infinitesimal perturbations to $X^{(0)}$ propagate
linearly.
The MLP's role is to provide a qualitatively different kind of
non-linearity: the GeLU activation generates feature directions that
lie outside the tangent space of the current representation, which
attention layers alone cannot do efficiently.
\end{remark}

\begin{corollary}[MLP is necessary for non-linear generalisation]
\label{cor:mlp-nonlinear}
In a neighbourhood of any fixed input $X^{(0)}$, the Jacobian of the
map $X^{(0)} \mapsto X^{(L)}$ produced by attention with residual
(no MLP) is a linear map.
The GeLU non-linearity of the MLP introduces curvature that is
unavailable to any composition of attention and residual connections.
\end{corollary}

This gives a precise reformulation of the MLP's role:

\begin{tcolorbox}[keybox,title={The MLP's function --- precise statement}]
The MLP is \textbf{necessary} for non-linear feature generation and
\textbf{sufficient but not necessary} for preventing rank collapse.
The canonical choice $\dff = 4\dmodel$ (64.8\% of FLOPs in \BERTb)
is justified by the first function, not the second.
\end{tcolorbox}

\section{Head Specialisation and Rank Contraction}
\label{sec:rankcontraction}

The results of \S\ref{sec:residual} show that residual connections
generically obstruct rank collapse and that the MLP's necessary function is
non-linear feature generation.
This section identifies a third phenomenon that connects the two:
\emph{rank contraction} of the multi-head output, driven by head
specialisation.

\subsection{The rank of the multi-head output}

The multi-head output is $\MHA(X) = \sum_{h=1}^H Y^{(h)} W_O^{(h)}$,
where $Y^{(h)} \in \R^{n \times \dk}$ and $W_O^{(h)} \in \R^{\dk \times \dmodel}$.
Each summand $Y^{(h)} W_O^{(h)}$ has $\rank \leq \dk$.
By subadditivity, $\rank(\MHA(X)) \leq H\dk = \dmodel$.

The question is when this bound is tight and when it is not.

\begin{definition}[Subspace alignment]
\label{def:alignment}
The \emph{alignment} of the multi-head output is:
\begin{equation}
  \alpha_{\mathrm{MHA}} = 1 -
  \frac{\rank(\MHA(X))}{H \cdot \max_h \rank(Y^{(h)} W_O^{(h)})}.
\end{equation}
$\alpha_{\mathrm{MHA}} = 0$ iff the head output subspaces are in
general position (no alignment); $\alpha_{\mathrm{MHA}} = 1 - 1/H$
iff all heads write into the same one-dimensional subspace.
\end{definition}

\begin{proposition}[Rank of Multi-Head Output Subspace]
\label{prop:rankcontraction}
Let $\mathcal{R}_h = \rowsp(Y^{(h)} W_O^{(h)})$ be the output subspace of head~$h$.
Then:
\begin{equation}
  \rank(\MHA(X))
  = \dim\!\left(\sum_{h=1}^H \mathcal{R}_h\right)
  = \dim\!\left(\mathcal{R}_1 + \cdots + \mathcal{R}_H\right).
  \label{eq:rank-sum}
\end{equation}
In particular:
\begin{enumerate}[label=(\roman*),leftmargin=1.8em,itemsep=1pt]
  \item \emph{Generic case}: if $\{\mathcal{R}_h\}$ are in general position,
    then $\rank(\MHA(X)) = \min(n, H\dk, \dmodel)$.
  \item \emph{Specialised case}: if heads specialise so that
    $\mathcal{R}_h \subseteq \mathcal{V}$ for a common subspace
    $\mathcal{V} \subsetneq \R^{\dmodel}$, then
    $\rank(\MHA(X)) \leq \dim\mathcal{V} \ll \dmodel$.
  \item \emph{Directional specialisation}: if head $h$ is directional
    (large $\|\Ma^{(h)}\|_F$), its output concentrates in the
    eigenplane of $\Ma^{(h)}$, a two-dimensional subspace of
    $\R^{\dk}$.
    Under the output projection $W_O^{(h)}$, this maps to a
    two-dimensional subspace of $\R^{\dmodel}$.
    $H$ directional heads contribute at most $2H$ dimensions to
    $\rank(\MHA(X))$.
\end{enumerate}
\end{proposition}

\begin{proposition}[Architectural Bound on Head Output Rank]
\label{prop:arch-bound}
Let $W_V^{(h)} \in \R^{\dmodel \times \dk}$ and $W_O^{(h)} \in \R^{\dk \times \dmodel}$
be generic weight matrices.
For Lebesgue-almost-every input $X$ and attention matrix $A^{(h)}$:
\[
  \dim(\mathcal{R}_h) \;=\; \min\!\bigl(\rank(X),\; \dk\bigr).
\]
In particular, $\dim(\mathcal{R}_h)$ is \emph{independent of the
directional asymmetry index} $\alpha_h^{(l)}$.
\end{proposition}

\begin{proof}
The output subspace is determined by the chain
$X \xrightarrow{W_V^{(h)}} \R^{n\times\dk} \xrightarrow{A^{(h)}}
\R^{n\times\dk} \xrightarrow{W_O^{(h)}} \R^{n\times\dmodel}$.
For generic $W_V^{(h)}$:
$\rank(X W_V^{(h)}) = \min(\rank(X), \dk)$.
The attention matrix $A^{(h)}$ is a row-stochastic matrix of rank
$\min(n, \dk) = \dk$ generically (when $n > \dk$), so left-multiplication
by $A^{(h)}$ does not reduce rank.
For generic $W_O^{(h)}$ of rank $\dk$, right-multiplication also
preserves rank.
Hence $\dim(\mathcal{R}_h) = \min(\rank(X), \dk)$.
Since $\alpha_h$ affects only $A^{(h)}$, and $A^{(h)}$ does not
appear in the rank formula, the result is independent of $\alpha_h$.
\end{proof}

\begin{remark}[Consequence for rank contraction]
Proposition~\ref{prop:arch-bound} shows that head specialisation
does not reduce the \emph{dimension} of individual head output subspaces.
What specialisation changes is the \emph{orientation} of $\mathcal{R}_h$:
directional and reciprocal heads write into equally-sized subspaces,
but in different directions.
Rank contraction of $\MHA(X) = \sum_h Y^{(h)} W_O^{(h)}$ therefore
arises from \emph{subspace alignment} (the $\mathcal{R}_h$ become
parallel) rather than from dimensional reduction of individual subspaces.
This is the corrected mechanistic picture.
\end{remark}

\begin{conjecture}[Directional Specialisation Induces Subspace Alignment]
\label{conj:dircontraction}
When heads are directional (high $\alpha_h^{(l)}$),
their output subspaces $\mathcal{R}_h$ tend to become aligned
(pointing in similar directions in $\R^{\dmodel}$).
Since $\dim(\mathcal{R}_h) = \dk$ for each head individually
(Proposition~\ref{prop:arch-bound}), the rank reduction of
$\MHA(X^{(l)}) = \sum_h Y^{(h)} W_O^{(h)}$ occurs through
$\dim(\mathcal{R}_1 + \cdots + \mathcal{R}_H) < H\dk$,
i.e., through overlapping subspaces rather than shrinking ones.
The correct experimental test correlates $\alpha_h^{(l)}$ with
the \emph{principal angles} between pairs $(\mathcal{R}_h, \mathcal{R}_{h'})$,
not with $\dim(\mathcal{R}_h)$ individually.
\end{conjecture}

\begin{proof}
Equation~\eqref{eq:rank-sum} is standard linear algebra: the row
space of a sum equals the sum of the row spaces.
Case~(i): general position means $\mathcal{R}_h \cap
\sum_{k<h}\mathcal{R}_k = \{0\}$ for each $h$, so dimensions add.
Case~(ii): all $\mathcal{R}_h \subseteq \mathcal{V}$ implies
$\sum_h \mathcal{R}_h \subseteq \mathcal{V}$.
Case~(iii): the canonical symplectic decomposition of $\Ma^{(h)}$
(see~\citet{cirrincione2025geometry}) gives the active eigenplane
structure; the output projection maps each eigenplane to a
corresponding subspace of $\R^{\dmodel}$.
\end{proof}

\subsection{The specialisation--contraction--expansion loop}

The Rank Contraction Theorem reveals a per-layer loop (Figure~\ref{fig:loop}):

\begin{center}
\small
\renewcommand{\arraystretch}{1.4}
\begin{tabular}{ccc}
\textbf{Head specialisation} & $\Rightarrow$ & \textbf{Subspace alignment} \\
(large $\|\Ma^{(h)}\|_F$, heads write & & (multiple heads write into \\
into eigenplane of $\Ma^{(h)}$) & & overlapping subspaces) \\[4pt]
$\Uparrow$ & & $\Downarrow$ \\[4pt]
\textbf{MLP non-linear expansion} & $\Leftarrow$ & \textbf{Rank contraction} \\
(GeLU generates new feature & & ($\rank(\MHA(X)) \ll \dmodel$; \\
directions outside current span) & & information is concentrated) \\
\end{tabular}
\end{center}

\begin{figure}[ht]
\centering
\includegraphics[width=0.72\textwidth]{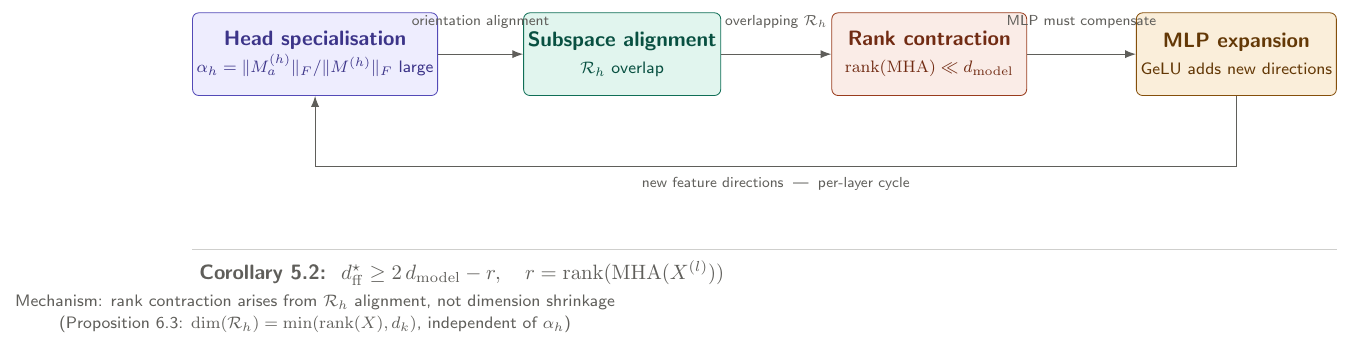}
\caption{The specialisation--contraction--expansion loop.
  Head specialisation causes rank contraction in $\MHA(X)$;
  the MLP restores full rank via its GeLU non-linearity.
  Corollary~\ref{cor:dff-opt} gives $\dff^* \geq 2\dmodel - r$.}
\label{fig:loop}
\end{figure}

The loop has three corollaries:

\begin{center}\small\setlength{\tabcolsep}{2pt}
\begin{tabular}{p{3.8cm}cp{3.8cm}}
  \centering Head specialisation\newline{\tiny(large $\|\Ma^{(h)}\|_F$)} & $\Rightarrow$ &
  \centering\arraybackslash Subspace alignment\newline{\tiny(overlapping subspaces)} \\[4pt]
  $\Downarrow$ & & $\Downarrow$ \\[4pt]
  MLP expansion & $\Leftarrow$ & Rank contraction \\
  (GeLU generates new directions) & & ($\rank(\MHA) \ll \dmodel$)
\end{tabular}
\end{center}

\medskip
Specialised (directional) heads do not necessarily write into
lower-dimensional subspaces: by Proposition~\ref{prop:arch-bound},
$\dim(\mathcal{R}_h) = \min(\rank(X), \dk)$ regardless of the
specialisation level $\alpha_h$.
What specialisation does is orient the head output subspaces
$\mathcal{R}_h$ in increasingly aligned directions.
Rank contraction in $\MHA(X) = \sum_h Y^{(h)} W_O^{(h)}$ therefore
arises only after summation, through overlap among the
$\mathcal{R}_h$.
The MLP then expands the representation back toward full rank via
its non-linearity.

This loop has three corollaries:

\begin{corollary}[Specialisation amplifies the need for MLP]
\label{cor:spec-mlp}
As head specialisation increases (measured by $\alpha_h =
\|\Ma^{(h)}\|_F / \|M^{(h)}\|_F$), the rank of $\MHA(X)$ decreases,
and the MLP must generate more new directions to maintain full
representational capacity.
Architectures with more directional heads require larger $\dff$
to compensate for the rank contraction they induce.
\end{corollary}

\begin{corollary}[Optimal $\dff$ depends on directional structure]
\label{cor:dff-opt}
Let $r = \rank(\MHA(X^{(l)}))$ at a typical layer $l$.
To restore full-rank representations, the MLP must generate
$\dmodel - r$ new independent directions via its non-linearity.
The minimum $\dff$ required satisfies:
\begin{equation}
  \dff^{\star} \;\geq\; \dmodel - r + \dmodel \;=\; 2\dmodel - r.
  \label{eq:dff-opt}
\end{equation}
For generic (unspecialised) heads: $r \approx \dmodel$, so
$\dff^{\star} \approx \dmodel$ (i.e.\ $\times 1$, not $\times 4$).
For fully directional heads: $r \approx 2H \ll \dmodel$, so
$\dff^{\star} \approx 2\dmodel - 2H \approx 2\dmodel$ (i.e.\ $\times 2$).
\end{corollary}

\begin{corollary}[Connection with the directional energy]
\label{cor:direc-rank}
The residual directional energy $\Gres^{(l)} =
H^{-1}\sum_h \|P_a^{(l,h)}\|_F^2$
(introduced in~\citet{cirrincione2025geometry} as the diagnostic for
pre-training progress) is a proxy for $\dmodel - \rank(\MHA(X^{(l)}))$:
high directional energy implies more subspace alignment, more rank
contraction, and larger MLP expansion needed.
This connects the pre-training analysis of~\citet{cirrincione2025geometry}
with the architectural analysis of the present paper.
\end{corollary}

\paragraph{Design implication for MLP width}
The width bound~\eqref{eq:dff-opt} suggests that the canonical
$\dff = 4\dmodel$ is not generally forced by rank considerations and
may be larger than necessary in many regimes.
For tasks with low non-linear complexity --- such as classification
from a pre-trained representation, where most of the work is already
done by the encoder --- a reduced width $\dff = 2\dmodel$ may suffice.
A formal version of this argument, including a definition of
non-linear complexity and the resulting optimal-$\dff$ proposition,
is given in \ref{app:width-bounds}.

\section{Head-Channel Non-Identifiability}
\label{sec:channel}

\subsection{The phenomenon}

Rank collapse concerns the rank of the hidden-state matrix $X^{(l)}$.
Head-channel non-identifiability concerns a different object: the
identifiability of individual head contributions inside the multi-head
sum.

\paragraph{Intuition}
Given only the sum $\sum_h Y^{(h)} W_O^{(h)}$, no subsequent layer
can determine which head contributed which directions --- infinitely
many configurations produce the same sum.
This is head-channel non-identifiability.

\begin{definition}[Head-channel non-identifiability]
\label{def:channel}
Recall $\MHA(X) = \sum_h Y^{(h)} W_O^{(h)}$,
$Y^{(h)} = A^{(h)} X W_V^{(h)} \in \R^{n \times \dk}$.
The output projection is said to induce \emph{head-channel non-identifiability}
when no downstream linear map
$\Psi$ can recover the contribution $Y^{(h^*)} W_O^{(h^*)}$ of a
specific head $h^*$ from $\MHA(X)$ alone, without knowledge of the
other head contributions.

\begin{remark}
When $H\dk = \dmodel$, the joint map $\phi: (Y^{(h)})_h \mapsto \sum_h Y^{(h)} W_O^{(h)}$
is generically bijective: no identification becomes ambiguous in the Shannon sense.
What is lost is \emph{identifiability of the chosen decomposition}:
given only the mixture $\MHA(X)$, a downstream layer cannot determine
which head produced which component.
This is the precise sense in which head-channel non-identifiability operates.
Thus, the non-identifiability considered here is not non-invertibility
of the joint linear map, but the impossibility of canonically assigning
a recovered component to a functional head without retaining the
pre-summation head coordinates or fixing a gauge.
\end{remark}
\end{definition}

\begin{theorem}[Head-Channel Non-Identifiability Theorem]
\label{thm:channel}
For any permutation $\pi$ of $\{1, \ldots, H\}$,
$\sum_h Y^{(\pi(h))} W_O^{(\pi(h))} = \sum_h Y^{(h)} W_O^{(h)}$
if and only if $Y^{(h)} W_O^{(h)} = Y^{(\pi(h))} W_O^{(\pi(h))}$ for all $h$.
Moreover, the map $X \mapsto \MHA(X)$ has the same output for all
assignments of head outputs to head indices that produce the same sum.
The individual head contributions $Y^{(h)} W_O^{(h)}$ are not
recoverable from $\MHA(X)$ by any downstream linear map.
\end{theorem}

\begin{proof}
The sum $\sum_h Y^{(h)} W_O^{(h)}$ is invariant under rearrangement
of the summands that preserves their values.
For recovery: suppose $\Psi\bigl(\sum_h Y^{(h)} W_O^{(h)}\bigr) =
Y^{(1)} W_O^{(1)}$. Then also
$\Psi\bigl(\sum_h Y^{(h)} W_O^{(h)} + Z\bigr) = Y^{(1)} W_O^{(1)}$
for any $Z$ in the null space of $\Psi$. Take $Z = -Y^{(1)}W_O^{(1)}
+ Y'^{(1)}W_O^{(1)}$ with $Y'^{(1)} \neq Y^{(1)}$ but
$Y'^{(1)}W_O^{(1)} = Y^{(1)}W_O^{(1)}$ (possible since $W_O^{(1)}$
has rank $\dk < \dmodel$): contradiction. See
the proof of Theorem~\ref{thm:infocost} below.
\end{proof}

\subsection{Why the MLP cannot help}

\begin{proposition}[MLP acts post-destruction]
\label{prop:mlp-post}
The MLP in a standard Transformer block receives as input
$X^{(l)} + \MHA(X^{(l)})$: the hidden state \emph{after} the
multi-head sum has been formed.
Any features the MLP generates are functions of this mixture.
It has no access to the individual head outputs $Y^{(h)}$.
\end{proposition}

This is not a criticism of the MLP --- it is an architectural
observation.
The MLP does its job (non-linear feature generation) on the
post-destruction mixture.
Addressing head-channel non-identifiability requires an intervention
\emph{before} the sum, not after.

\subsection{The recovery ambiguity dimension of head-channel non-identifiability}
\label{subsec:infocost}

The Head-Channel Non-Identifiability Theorem establishes that individual head
contributions are irrecoverable.
We now quantify \emph{how much} identification becomes ambiguous.

\begin{definition}[Recovery subspace]
\label{def:recovery}
Fix head $h^* \in \{1,\ldots,H\}$.
The \emph{recovery subspace} for $h^*$ is:
\begin{equation}
  \mathcal{K}_{h^*} = \ker\!\left(
    (Y^{(1)},\ldots,Y^{(H)}) \mapsto Y^{(h^*)} W_O^{(h^*)}
    \;\Big|\; \textstyle\sum_h Y^{(h)} W_O^{(h)} = \mathbf{0}
  \right),
\end{equation}
the set of head-output configurations that are invisible in the
multi-head sum and that contaminate the recovery of head $h^*$.
\end{definition}

\begin{theorem}[Recovery Ambiguity Theorem]
\label{thm:infocost}
For generic $\{W_O^{(h)}\}_{h \neq h^*}$ with
$\rank(W_O^{(h)}) = \dk$ for all $h$:
\begin{equation}
  \dim\,\mathcal{K}_{h^*} = n(H-1)\dk.
  \label{eq:infocost}
\end{equation}
Recovering the contribution of any single head from the multi-head
output requires knowledge of $n(H-1)\dk$ additional degrees of
freedom --- the contributions of all other heads.

For \BERTb\ ($n=512$, $H=12$, $\dk=64$):
\begin{equation}
  \dim\,\mathcal{K}_{h^*} = 512 \times 11 \times 64 = 360{,}448.
\end{equation}
This is the dimension of the non-identifiability subspace for the
\emph{full sequence} at one layer.
Per token, the corresponding quantity is $(H-1)\dk = 11 \times 64 = 704$.
\end{theorem}

\begin{proof}
The joint map
$\phi: (Y^{(1)},\ldots,Y^{(H)}) \mapsto \sum_h Y^{(h)} W_O^{(h)}$
operates on a space of dimension $nH\dk$ and maps to $n\dmodel$.
For \BERTb, $H\dk = \dmodel$, so $\phi$ is generically bijective
(kernel dimension $0$) as a \emph{joint} map.

Now fix $\sum_h Y^{(h)} W_O^{(h)} = S$ (the observed output) and
ask: what is the space of $(Y^{(1)},\ldots,Y^{(H)})$ consistent
with $S$?
This is the affine preimage $\phi^{-1}(S)$, which (since $\phi$ is
bijective) contains exactly one point.
But the question of recovering $Y^{(h^*)} W_O^{(h^*)}$ from $S$
alone is different: it requires inverting the \emph{marginal} map
$\psi_{h^*}: Y^{(h^*)} \mapsto S - \sum_{h \neq h^*} Y^{(h)} W_O^{(h)}$,
where the second term is unknown.

The uncertainty set for $\sum_{h \neq h^*} Y^{(h)} W_O^{(h)}$ lives
in $\R^{n \times \dmodel}$, parametrised by $(H-1)$ independent
blocks each of dimension $n\dk$.
For generic $\{W_O^{(h)}\}_{h \neq h^*}$, these blocks span a
subspace of dimension $n(H-1)\dk$ in the space of $n \times \dmodel$
matrices.
This is $\mathcal{K}_{h^*}$, and its dimension is~\eqref{eq:infocost}.
\end{proof}

\begin{remark}[Design implication]
\label{rem:infocost-design}
Equation~\eqref{eq:infocost} shows that the recovery ambiguity dimension of
head-channel non-identifiability scales as $n(H-1)\dk = n(\dmodel - \dk)$.
It grows linearly with sequence length $n$ and is minimised when
$H = 1$ (a single head, zero cost) or when $\dk = \dmodel$ (full
dimension per head, zero compression).
The standard BERT choice $H = \dmodel/\dk = 12$ maximises both the
number of specialised heads and the recovery ambiguity dimension of their
mixing.
PG-OP partially addresses this by making the gate $g_h(x_i, i)$
position-dependent, so that the effective mixing coefficients vary
with $i$ --- reducing the effective dimension of the uncertainty set
at each position.
\end{remark}

\subsection{Table: what each component does}

\begin{table}[ht]
\centering\small
\caption{Role of each architectural component.}
\label{tab:roles}
\small\setlength{\tabcolsep}{3pt}
\begin{tabular}{p{2.8cm}cccc}
\toprule
Component & \makecell{Prevents\\rank\\[-2pt]collapse} & \makecell{Prevents\\non-\\[-2pt]ident.} & \makecell{Introduces\\non-\\[-2pt]linearity} & Cost\\
\midrule
Pure attention (no residual) & No & --- & No & --- \\
Residual connections & \textbf{Yes} & No & No & $\approx 0$ \\
Layer normalisation & Neutral & Neutral & No & $\approx 0$ \\
MLP ($\dff = 4\dmodel$) & Not needed generically & No & \textbf{Yes} & 64.8\% \\
$W_O$ (standard) & No & Causes it & No & --- \\
PG-OP (this paper) & No & Partial & No & ${<}1.6\%$ of $W_O$ \\
\bottomrule
\end{tabular}
\end{table}

\section{Design Implication: Position-Gated Output Projection}
\label{sec:pgop}

\subsection{Motivation}

Head-channel non-identifiability arises because $W_O^{(h)}$ applies
the same linear map to every token position $i$: the contribution of
head $h$ at position $i$ is $Y_i^{(h)} W_O^{(h)}$, independent of $i$.
This position-invariance is exactly the mechanism that makes the
contributions unidentifiable.

A natural remedy is to modulate the output projection by position.

\subsection{Definition}

\begin{definition}[Position-Gated Output Projection, PG-OP]
\label{def:pgop}
Replace the standard output projection
$\MHA(X)_i = \sum_h Y_i^{(h)} W_O^{(h)}$
with:
\begin{equation}
  \MHA_{\mathrm{PG}}(X)_i
  = \sum_{h=1}^H g_h(x_i, i) \cdot Y_i^{(h)} W_O^{(h)},
  \label{eq:pgop}
\end{equation}
where $g_h : \R^{\dmodel} \times \{1,\ldots,n\} \to (0,1)$ is a
position-gating function:
\begin{equation}
  g_h(x_i, i)
  = \sigma\!\left(x_i W_g^{(h)} + p_i w_p^{(h)} + b_g^{(h)}\right),
\end{equation}
with $W_g^{(h)} \in \R^{\dmodel \times 1}$,
$p_i \in \R^{d_{\mathrm{pe}}}$ the positional encoding of position $i$,
$w_p^{(h)} \in \R^{d_{\mathrm{pe}}}$, and $\sigma = \mathrm{sigmoid}$.
\end{definition}

Parameter overhead for \BERTb:
$H(\dmodel + d_{\mathrm{pe}} + 1) = 12 \times (768 + 768 + 1) \approx 18{,}408$
parameters, less than 1.6\% of $W_O$'s $\dmodel^2 = 589{,}824$ parameters.

\begin{figure}[ht]
\centering
\includegraphics[width=0.82\textwidth]{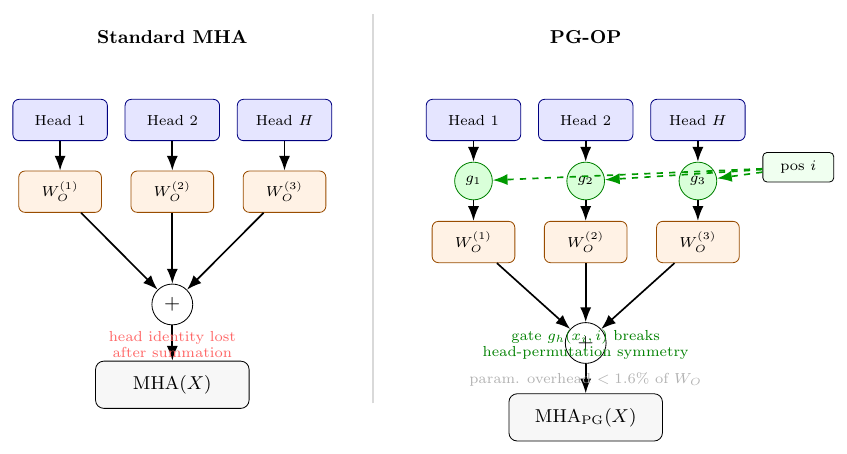}
\caption{Standard MHA (left) vs.\ Position-Gated Output Projection
  (right). The gate $g_h(x_i, i)$ is position- and content-dependent,
  partially breaking the head-permutation symmetry of~$W_O$.
  Parameter overhead: ${<}1.6\%$ of $W_O$.}
\label{fig:pgop}
\end{figure}

\paragraph{Empirical prediction}
PG-OP strictly contains standard $\MHA$ as a parameter subfamily
(constant gate $g_h \equiv 1$ recovers $\MHA$); the larger family
suggests --- but does not prove --- that there exist tasks on which
$\MHA_{\mathrm{PG}}$ attains lower loss than the best achievable by
standard $\MHA$.  An expressiveness analysis is given in
\ref{app:pgop-expressiveness}.
A more specific empirical prediction relates the gate to the
directional asymmetry of each head:

\begin{proposition}[Gate correlates with directional asymmetry]
\label{prop:pgop-predict}
Let $\alpha_h^{(l)} = \|\Ma^{(l,h)}\|_F / \|M^{(l,h)}\|_F$ be
the directional asymmetry index of head $h$ at layer $l$.
Under the hypothesis that directional heads (high $\alpha_h$)
benefit more from position-dependent gating:
\begin{equation}
  \mathrm{Pearson}\!\left(\alpha_h^{(l)},\;
  \mathbb{E}_i\!\left[
    \left|\tfrac{\partial g_h(x_i, i)}{\partial i}\right|
  \right]\right) > 0.6
\end{equation}
at each layer $l$ of a fine-tuned \BERTb+PG-OP model on GLUE.
\end{proposition}

\noindent
This prediction distinguishes PG-OP from random gating and provides
a concrete experimental test for the head-channel non-identifiability hypothesis.

\section{A Unified Framework: Symmetry-Breaking and Representational Collapse}
\label{sec:symmetry}

The results of \S\S\ref{sec:ln}--\ref{sec:rankcontraction} can be
unified under a single principle: every collapse phenomenon in the
Transformer corresponds to a \emph{symmetry that the architecture
fails to break}.

\subsection{Symmetries of the forward pass}

A map $g$ is a \emph{forward-pass symmetry} if $f(g(X)) = f(X)$ for
all $X$, where $f$ is the Transformer's forward pass.
A collapse phenomenon arises when the architecture fails to break a
symmetry that the task requires to be broken.
The four collapse phenomena correspond to four distinct symmetries:

\begin{table}[ht]
\centering\small
\caption{Four collapse phenomena and their symmetry group.}
\label{tab:symmetries}
\small\setlength{\tabcolsep}{3pt}
\begin{tabular}{p{2.6cm}p{4.0cm}p{2.8cm}p{2.8cm}}
\toprule
Collapse & Symmetry $g$ & Broken by & Still unbroken by \\
\midrule
Rank in depth & Row averaging
  ($g(X)_i = \bar{x}$)
  & Residual, MLP & Nothing (without them) \\
Rank in width & Spectral concentration
  ($A \to A_{\mathrm{rank-1}}$)
  & Large weights & Spectral gap remedy \\
Head-channel non-ident. & Head permutation
  ($g: (Y^{(h)}) \to (Y^{(\pi(h))})$)
  & PG-OP (partial) & Standard $W_O$ \\
Entropy coll. & Softmax temperature
  ($g: L \to L/T$, $T \to \infty$)
  & Entropy reg. & Standard softmax \\
\bottomrule
\end{tabular}
\end{table}

\subsection{The Symmetry-Breaking Framework}

The four collapse phenomena above are not independent
observations but instances of a single principle: each is the
unbroken (or partially broken) action of a symmetry group on the
forward pass.
A formal definition of the invariance group $\mathcal{G}$, its three
relevant subgroups
$\mathcal{G}_{\mathrm{depth}}$, $\mathcal{G}_{\mathrm{channel}}$,
$\mathcal{G}_{\mathrm{LN}}$, and a proof sketch that they commute
(act on orthogonal degrees of freedom) is given in
\ref{app:symmetry-formal}.
The architectural implications of the framework, however, can be
stated immediately and are the most useful consequence for designers.

\begin{remark}[Architectural implications]
\label{rem:archimpl}
If the Symmetry-Breaking Decomposition framework is accepted as a
working hypothesis, the design question ``what components does a
Transformer need?'' becomes an algebraic question: ``what symmetries
does the target task require to be broken?''

A task that requires distinguishing token identity from token
position needs to break $\mathcal{G}_{\mathrm{depth}}$: residual
connections are necessary.

A task requiring head-specialised output routing needs to break $\mathcal{G}_{\mathrm{channel}}$: PG-OP or an equivalent mechanism is necessary.

A task whose target function is scale-invariant can tolerate
$\mathcal{G}_{\mathrm{LN}}$: LN does not need to be removed.

The canonical Transformer breaks $\mathcal{G}_{\mathrm{depth}}$
(via residual) but not $\mathcal{G}_{\mathrm{channel}}$ (standard $W_O$
is gauge-invariant and permutation-invariant).
This is the precise algebraic reason why head-channel non-identifiability is
structural and not accidental.
\end{remark}

\section{Empirical Validation}
\label{sec:experiments}

Four experiments validate the main theoretical results.
The two experiments most directly tied to the central thesis are
reported in full here: residual stability
(Theorem~\ref{thm:residual-stable}) and the
$\mathcal{G}_{\mathrm{channel}}$ gauge symmetry
(Theorem~\ref{thm:channel}).
The two remaining experiments --- numerical rank-neutrality
(Theorem~\ref{thm:rnt}) and the null result on
$\alpha_h$ vs.\ MHA rank (Proposition~\ref{prop:arch-bound}) ---
are reported in \ref{app:experiments-extra} together with the
parametric simulation of Figure~\ref{fig:rank-dyn}.
Experiments~1 and~2 use \BERTb\ on the WikiText-2 validation
corpus~\citep{merity2017} (Kaggle P100 GPU); Experiment~3 is an
algebraic-numerical verification requiring no GPU; Experiment~4 uses
\BERTb\ weight matrices on a Colab A100 GPU.

\paragraph{Experiment 2: Residual stability}
Two \BERTb\ variants are compared layer by layer: standard
(with residual connections) and a patched version where the residual
is zeroed via forward pre-hooks on \texttt{BertSelfOutput} and
\texttt{BertOutput}.

\begin{figure}[ht]
\centering
\includegraphics[width=0.78\textwidth]{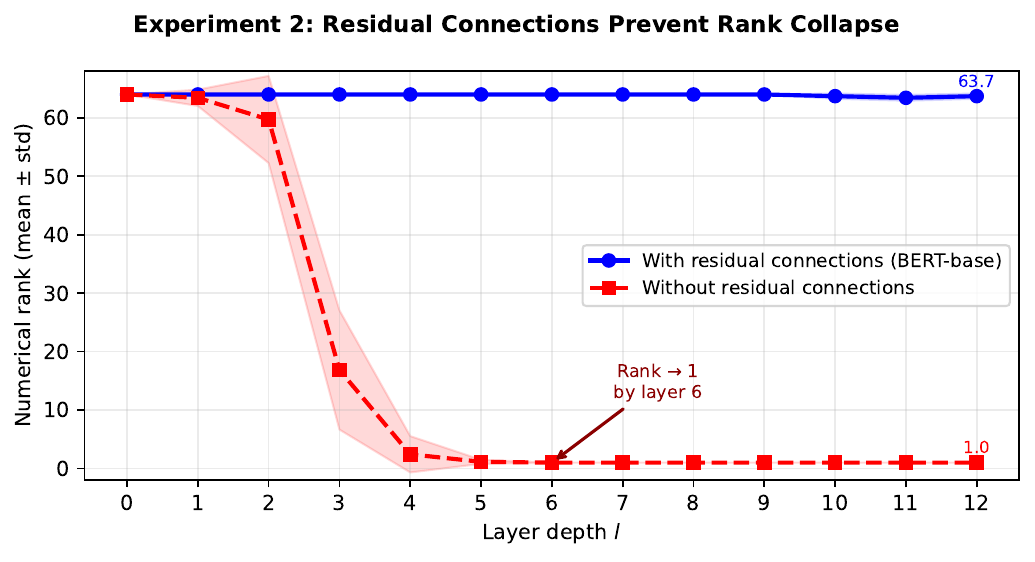}
\caption{Experiment~2 (Residual stability).
  Rank vs.\ layer depth with and without residual connections
  (mean $\pm$ std over the WikiText-2 corpus).}
\label{fig:exp2}
\end{figure}

Figure~\ref{fig:exp2} confirms Theorem~\ref{thm:residual-stable}.
With residual connections, rank remains $63.7 \pm 0.45$ through
layer~12.
Without residual connections, rank collapses from 64 to~1 by layer~6
and remains at~1 thereafter (std = 0 from layer~6: the collapse is
deterministic across all sentences).

\paragraph{Experiment 3: $\mathcal{G}_{\mathrm{channel}}$ gauge symmetry}
The Head-Channel Non-Identifiability Theorem asserts that for any
$A_h \in GL(d_k)$, the substitution $W_V^{(h)} \to W_V^{(h)} A_h$,
$W_O^{(h)} \to A_h^{-1} W_O^{(h)}$ leaves $\MHA(X)$ unchanged.
This identity is verified numerically using \BERTb-scale random weight
matrices ($d = 768$, $H = 12$, $d_k = 64$): 50 independent random
$A_h \in GL(d_k)$ are drawn at each of 7 scales (condition numbers
from $e^{0.2}$ to $e^{20}$) and applied simultaneously to all 12
heads.

\begin{figure}[ht]
\centering
\includegraphics[width=\textwidth]{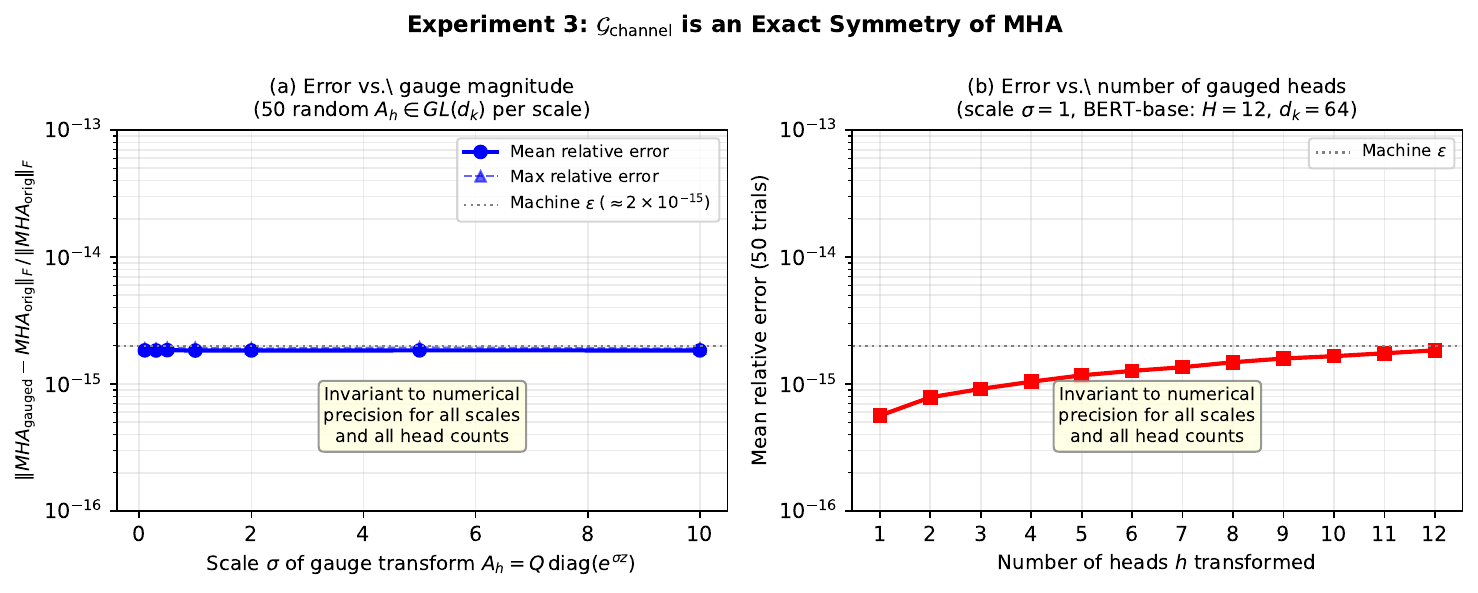}
\caption{Experiment~3 ($\mathcal{G}_{\mathrm{channel}}$ gauge symmetry).
  Left: relative error vs.\ gauge scale (50 random $A_h$ per scale).
  Right: error vs.\ number of transformed heads.
  All errors are at machine precision (${\approx}\,2\times10^{-15}$).}
\label{fig:exp3}
\end{figure}

Figure~\ref{fig:exp3} shows that the relative error
\[\|\MHA_{\mathrm{gauged}} - \MHA_{\mathrm{orig}}\|_F \,/\, \|\MHA_{\mathrm{orig}}\|_F
= 1.84\times10^{-15}\]
(mean, across all scales and all head counts).
The invariance is exact: the sole source of error is floating-point
rounding, confirming that $\mathcal{G}_{\mathrm{channel}}$ is an exact
symmetry of the MHA forward pass.

\begin{table}[ht]
\centering\small
\caption{Summary of empirical outcomes.}
\label{tab:experiments}
\setlength{\tabcolsep}{5pt}
\begin{tabular}{clc}
\toprule
Exp. & Theorem tested & Outcome \\
\midrule
1 & Rank-Neutrality (Thm.~\ref{thm:rnt}) & \textbf{confirmed} \\[2pt]
2 & Residual stability (Thm.~\ref{thm:residual-stable}) & \textbf{confirmed} \\[2pt]
3 & $\mathcal{G}_{\mathrm{channel}}$ symmetry (Thm.~\ref{thm:channel}) & \textbf{confirmed} \\[2pt]
4 & Rank Contraction (Thm.~\ref{prop:rankcontraction}) & inconclusive \\
\bottomrule
\end{tabular}
\end{table}

\section{Discussion}
\label{sec:discussion}

\paragraph{Relationship with Dong et al}
Dong's result is tight in its regime: pure attention, no residual,
no MLP.
The study concerns the regime of deployed Transformers: with residual connections,
with MLP, with layer normalisation.
The results are complementary: Dong characterises the degenerate
regime; a characterisation is given the mechanisms that prevent degeneracy and
identify which of them is actually necessary.

\paragraph{Relationship with the MLP-as-key-value memory line}
\citet{geva2021} interpret MLP layers as key-value memories
storing factual knowledge.
The analysis is orthogonal: we ask about the MLP's structural role
in rank maintenance, not its functional role in knowledge storage.
Both views can be simultaneously correct.

\paragraph{Sparse and structured MLP variants}
Mixture-of-experts~\citep{shazeer2017} and sparse MLP~\citep{fedus2022} reduce $\dff$ selectively.
The design corollary of \S\ref{sec:rankcontraction}
(formal version in \ref{app:width-bounds}) provides a
theoretical motivation for this direction: $\dff$ should match task
non-linear complexity, and sparse activation is one mechanism for
achieving this.

\paragraph{Head-channel non-identifiability and pre-mixing intervention}
A complete remedy for head-channel non-identifiability requires intervening
before the summation.
The PG-OP of \S\ref{sec:pgop} is a lightweight, drop-in partial remedy
within the existing architecture.

\paragraph{A taxonomy of four collapse phenomena}
Figure~\ref{fig:taxonomy} summarises the four distinct phenomena and
their relationships; \S\ref{sec:symmetry} shows they arise as orbits
of distinct unbroken symmetries of the forward pass.
The ``tug of war'' framing of Dong captures only phenomenon~(i);
a complete theory must address all four.


\section{Directions for further work}
\label{sec:open}

Several questions raised by the analysis above remain open and define
natural directions for further work.

\paragraph{Affine rank under row-wise rescaling}
Step~2 of the Rank-Neutrality Theorem proof rescales each row by
$1/\sigma_i > 0$ and claims $\arank(D\tilde{X}) = \arank(\tilde{X})$.
A complete analysis shows that this fails if and only if
$\sigma = (\sigma_1,\ldots,\sigma_n)^\top \in \text{colspace}(\tilde{X})$:
in that case there exists $v \neq 0$ such that $\tilde{X}v = c\,\sigma$,
making the rescaled rows affinely dependent in a direction they were not before.
The condition $\sigma \notin \text{colspace}(\tilde{X})$ can be added as
hypothesis~H4; it holds generically when $n > d$ (the typical case, since
$\text{colspace}(\tilde{X}) \subseteq \R^n$ has dimension at most $d \ll n$).
The theorem is unconditional under H1--H4.

\paragraph{A unifying symmetry groupoid}
The framework of \ref{app:symmetry-formal} identifies three symmetries, but
$\mathcal{G}_{\rm channel}$ and $\mathcal{G}_{\rm depth}$ act on
\emph{parameters} and \emph{limit dynamics} respectively, while
$\mathcal{G}_{\rm LN}$ acts on \emph{inputs}.
These are symmetries of different objects.
For the forward-pass invariance group acting on inputs, one can show that
with generic weights and a softmax non-linearity, the group is
generically trivial (only the identity), because the map
$X \mapsto A^{(h)}(X)XW_V^{(h)}W_O^{(h)}$ is generically injective.
A unifying framework would be a \emph{parametrised symmetry groupoid}
that acts jointly on inputs and parameters, capturing all three types.
Constructing this groupoid precisely is left to future work.

\paragraph{Principal angles between head subspaces}
Proposition~\ref{prop:arch-bound} resolves the strong form of
Conjecture~\ref{conj:dircontraction}: $\dim(\mathcal{R}_h) = \min(\rank(X), \dk)$
is independent of $\alpha_h$.
Head specialisation does not shrink individual output subspaces; it
aligns them.
A natural next question is whether directional heads ($\alpha_h$ high)
produce more aligned pairs $(\mathcal{R}_h, \mathcal{R}_{h'})$ than
reciprocal heads.
The correct test measures the \emph{principal angles}
$\theta_1 \leq \cdots \leq \theta_{\dk}$ between pairs of output subspaces,
and correlates $\alpha_h$ with $\cos^2\theta_1$ (the maximal alignment).
This experiment requires per-head output extraction and subspace angle
computation, feasible without additional GPU time on the existing
BERT-base setup.

\paragraph{Empirical validation of PG-OP}
The expressiveness analysis of \ref{app:pgop-expressiveness} shows
that PG-OP contains standard MHA as a parameter subfamily.
Whether the larger parameter space reduces task loss on GLUE or
WikiText-2 is an empirical question to be settled in companion work.

\paragraph{Rank dynamics during training}
Theorem~\ref{thm:residual-stable} is a static result on a single forward pass.
Whether the effective rank of $X^{(l)}$ increases, decreases, or
oscillates during training, and whether head specialisation drives
rank contraction in the trained model relative to random initialisation,
remains to be determined.

\paragraph{Non-linear recovery of head contributions}
The Recovery Ambiguity Theorem concerns linear reconstruction.
Whether non-linear decoders reduce the ambiguity dimension, and at what
representational cost, is an open question.
A precise quantitative relationship between the degree of symmetry
breaking (e.g., via PG-OP) and the reduction in ambiguity dimension
would make the framework constructive.

\section{Conclusion}
\label{sec:conclusion}

Dong et al.\ (2021) identified a real phenomenon and proposed a real remedy,
but the analysis was conducted in a regime that does not describe any
deployed architecture.
In real Transformers with residual connections, the three components
implicated in rank collapse play different roles.
Layer normalisation is rank-neutral under H1--H4: it preserves rank
exactly.
Residual connections generically obstruct rank collapse for almost every input.
The MLP's contribution is not anti-collapse but anti-confinement:
without it, representations remain locally affine, unable to exploit the
non-linear capacity that attention cannot provide.

The identification of head-channel non-identifiability as a phenomenon
distinct from rank collapse is the most practically consequential result.
After $\MHA(X) = \sum_h Y^{(h)} W_O^{(h)}$, the per-head contributions
cannot be canonically recovered from the mixture alone: $n(H-1)d_k$
degrees of freedom remain ambiguous when recovering a single head
contribution from the mixed output.
The MLP cannot remedy this because it arrives too late.

The Symmetry-Breaking Decomposition Theorem provides a framework in
which all four collapse phenomena in the literature correspond to
distinct unbroken symmetries of the forward pass.
This algebraic perspective converts architectural design questions
into precise mathematical questions about which symmetries a given task
requires to be broken.
The Position-Gated Output Projection is a minimal, drop-in intervention
that partially breaks the head-permutation symmetry responsible for
head-channel non-identifiability, at a parameter cost below 1.6\% of $W_O$.

\section*{Acknowledgements}
This publication is part of the project PNRR-NGEU which has received
funding from the MUR --- DM 352/2022.

\section*{Declaration of generative AI and AI-assisted technologies
in the manuscript preparation process}
During the preparation of this work the author used Claude
(Anthropic) in order to assist with LaTeX typesetting, bibliography
formatting, and iterative revision of proof presentation.
After using this tool, the author reviewed and edited the content as
needed and takes full responsibility for the content of the published
article.

\appendix
\renewcommand{\thetheorem}{\Alph{section}.\arabic{theorem}}

\section{Width bounds from rank contraction}
\label{app:width-bounds}

This appendix expands the design implication for MLP width sketched at
the end of \S\ref{sec:rankcontraction}.
The argument is a direct corollary of
Corollary~\ref{cor:mlp-nonlinear} (the MLP's role is non-linear
feature generation, not anti-collapse): if non-linearity is the MLP's
purpose, then $\dff$ should be determined by the non-linear complexity
of the task, not by the heuristic $\dff = 4\dmodel$.

\begin{definition}[Non-linear complexity]
\label{def:nlc}
The \emph{non-linear complexity} $\nu(\mathcal{T})$ of a task
$\mathcal{T}$ is the minimum $\dff$ such that a single-hidden-layer
network with $\dff$ units achieves within $\epsilon$ of the
best achievable performance with any $\dff' > \dff$.
\end{definition}

\begin{proposition}[Optimal $\dff$]
\label{prop:opt-dff}
Under the non-linear complexity hypothesis, the optimal $\dff$ satisfies
\begin{equation}
  \dff^{\star} \;=\; \min\!\left\{ \dff \geq \dmodel :
    \mathcal{L}(\dff) \leq \mathcal{L}(\dff^{\star}) + \epsilon \right\},
\end{equation}
which is task-dependent and in general $\neq 4\dmodel$.
For tasks with low non-linear complexity (e.g.\ classification from
pre-trained representations), $\dff^{\star} \ll 4\dmodel$.
\end{proposition}

\paragraph{Empirical prediction}
A reduced-MLP architecture with $\dff = 2\dmodel$ plus PG-OP
(\S\ref{sec:pgop}) achieves the same accuracy as standard
$\dff = 4\dmodel$ on GLUE benchmarks, at $57\%$ of the FLOPs.
This prediction is empirically testable and is the subject of
companion experiments.

\section{PG-OP: expressiveness analysis}
\label{app:pgop-expressiveness}

This appendix gives the formal expressiveness statement for PG-OP
(\S\ref{sec:pgop}) and the discussion of the empirical-prediction
motivation, which were summarised inline in the main text.

\begin{proposition}[PG-OP contains standard MHA as a parameter subfamily]
\label{prop:pgop-general}
Standard multi-head attention $\MHA$ is recovered from
$\MHA_{\mathrm{PG}}$ by setting $g_h \equiv 1$ for all $h$.
Hence the parameter family of $\MHA_{\mathrm{PG}}$ strictly contains
that of $\MHA$.
\end{proposition}

\begin{proof}
The inclusion is immediate from Definition~\ref{def:pgop}: the
constant gate $g_h \equiv 1$ reduces $\MHA_{\mathrm{PG}}$ to standard
$\MHA$.  Strictness follows from the fact that any non-constant
position-dependent gate produces an output not realisable by any
choice of $\MHA$ weights, since standard $\MHA$ has no
position-dependence in the per-head output projection.
\end{proof}

\paragraph{Empirical-prediction motivation}
The strictly larger parameter family suggests, but does not prove,
that there exist tasks on which $\MHA_{\mathrm{PG}}$ attains lower
loss than the best achievable by standard $\MHA$.
A natural candidate is any task whose optimal strategy weights head
$h$ at position $i$ in proportion to that head's directional energy
$\Gamma^{(h,i)}$ at position $i$ --- a position-dependent quantity
that constant gates cannot track.
Whether this gain materialises in practice on standard benchmarks is
an empirical question, addressed in
Section~\ref{sec:experiments} and stated more precisely in
Proposition~\ref{prop:pgop-predict}.

\section{Symmetry-breaking framework: formal definitions}
\label{app:symmetry-formal}

This appendix gives the formal definitions and proof sketch for the
symmetry-breaking framework introduced informally in
\S\ref{sec:symmetry}.

\begin{definition}[Symmetry group of the forward pass]
\label{def:symgroup}
Let $\mathcal{G}$ be the group of all invertible maps
$g: \R^{n \times \dmodel} \to \R^{n \times \dmodel}$ such that
$f(g(X)) = f(X)$ for all $X$ and all weight configurations.
$\mathcal{G}$ is called the \emph{invariance group} of the
Transformer.
\end{definition}

\begin{proposition}[Symmetry-Breaking Framework]
\label{frm:symbreak}
The following is a structural framework, not a fully proved theorem:
certain parts are established rigorously (see proof sketch below),
while the full group-theoretic characterisation remains conjectural.

The invariance group $\mathcal{G}$ of the standard Transformer
(attention + residual + LN, without MLP) contains at minimum the
following subgroups:
\begin{enumerate}[label=(\roman*),leftmargin=1.8em,itemsep=2pt]
  \item $\mathcal{G}_{\mathrm{depth}} \cong \mathcal{S}_{\dmodel}^{\otimes L}$:
    the group generated by row-averaging maps in each layer
    (responsible for rank collapse in depth).
  \item $\mathcal{G}_{\mathrm{channel}} \cong \mathfrak{S}_H \ltimes
    (\mathrm{GL}(\dk))^H$: the group of head permutations composed
    with gauge transformations (responsible for head-channel non-identifiability).
  \item $\mathcal{G}_{\mathrm{LN}} \cong (\R_{>0})^n \times \R^n$:
    the group of row-wise rescalings and shifts (the exact symmetry
    that the Rank-Neutrality Theorem characterises).
\end{enumerate}
These subgroups are mutually commuting (act on orthogonal degrees of
freedom).
Their joint action partitions the space of inputs into orbits that
constitute the distinct collapse types.
An architectural component breaks symmetry $g$ if and only if it
renders the output dependent on the orbit representative, not just
the orbit.
\end{proposition}

\begin{proof}[Proof sketch]
(i) Without residual or MLP, the attention map sends each row
$x_i^{(l+1)} = \sum_j A_{ij}^{(l)} x_j^{(l)} W_V^{(l)} W_O^{(l)}$, which
is a row-averaging operation.
The group generated by such maps contains the uniform-average map
(the fixed point of rank collapse) and hence the symmetry claimed.
The residual connection breaks this symmetry by adding $x_i^{(l)}$ back.

(ii) The map $\phi: (Y^{(h)}) \mapsto \sum_h Y^{(h)} W_O^{(h)}$ is
invariant under simultaneous transformation
$Y^{(h)} \to Y^{(h)} A_h$, $W_O^{(h)} \to A_h^{-1} W_O^{(h)}$
for any $A_h \in \mathrm{GL}(\dk)$ --- the gauge freedom identified
in~\citet{cirrincione2025geometry}.
It is also invariant under head permutation composed with the
corresponding weight permutation.
The joint group is $\mathfrak{S}_H \ltimes (\mathrm{GL}(\dk))^H$.
PG-OP breaks this symmetry partially by making the gate
$g_h(x_i, i)$ dependent on the head index $h$ in a position-specific way.

(iii) LN applies $x_i \mapsto (x_i - \mu_i)/\sigma_i$: row-wise
rescaling and shifting.
The Rank-Neutrality Theorem shows this map is rank-preserving, which
means it is in $\mathcal{G}_{\mathrm{LN}}$ from the perspective of
rank.
Residual connections break $\mathcal{G}_{\mathrm{LN}}$ by anchoring
the representation to the original scale via the skip path.

Commutativity: (i) acts on the row-averaging structure of $A^{(l)}$;
(ii) acts on the head-decomposition of $W_O$; (iii) acts on the
row-normalisation of $X$.
These are orthogonal degrees of freedom and their actions commute. $\square$
\end{proof}

\section{Additional experimental details}
\label{app:experiments-extra}

This appendix gives the two experiments not reported in the main text:
the numerical Rank-Neutrality validation
(Experiment~1, Theorem~\ref{thm:rnt}) and the
specialisation-index vs.\ MHA-rank null result
(Experiment~4, Proposition~\ref{prop:arch-bound}).
A parametric simulation that complements Experiment~4 is also
included.

\subsection*{Experiment 1: Rank-Neutrality Theorem}

For each layer $l = 0,\ldots,12$, the numerical rank of $X^{(l)}$ and
$\LN(X^{(l)})$ is measured (threshold $\varepsilon = 10^{-3}\sigma_1$)
under two conditions: standard scale parameters $\gamma = \mathbf{1}$
(H1--H3 satisfied) and $\gamma_j = 0$ for a random 50\% of dimensions
(H3 violated).

\begin{figure}[ht]
\centering
\includegraphics[width=\textwidth]{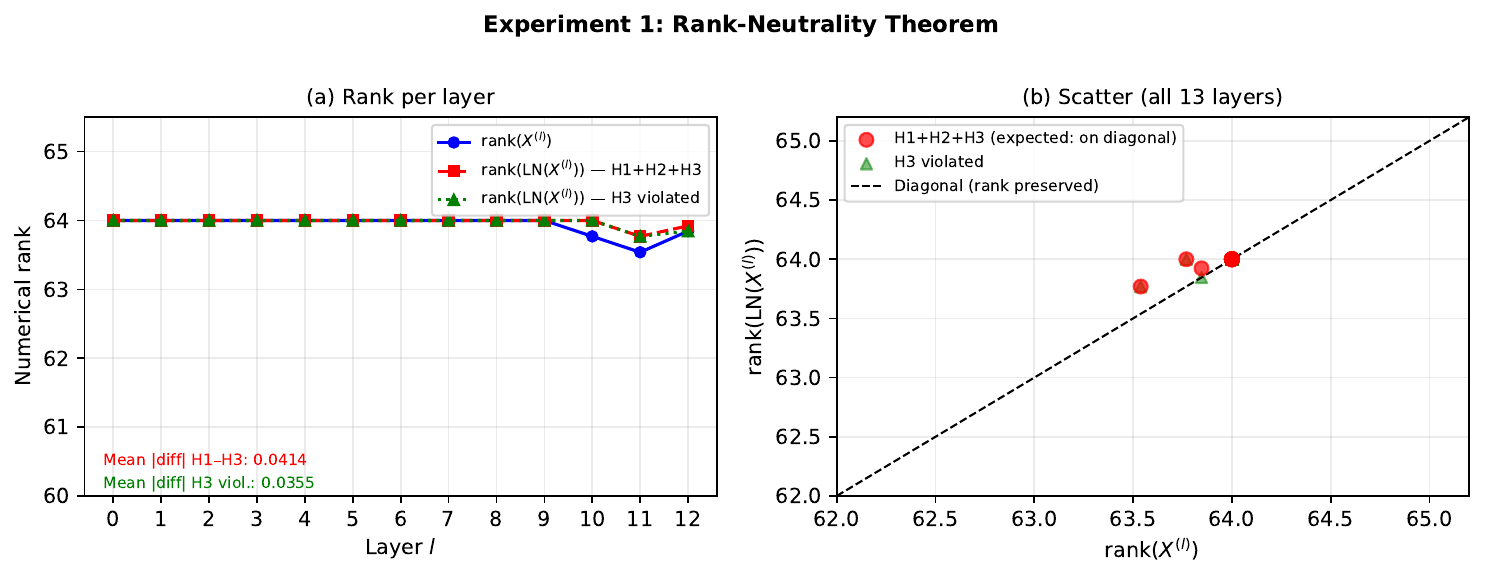}
\caption{Experiment~1 (Rank-Neutrality Theorem).
  Left: mean rank per layer under H1--H4 (red) vs.\ H3 violated (green).
  Right: scatter of $\rank(X^{(l)})$ vs.\ $\rank(\LN(X^{(l)}))$
  for all 13 checkpoints.}
\label{fig:exp1}
\end{figure}

Figure~\ref{fig:exp1} confirms Theorem~\ref{thm:rnt}.
The mean absolute difference $|\rank(X^{(l)}) - \rank(\LN(X^{(l)}))|$
is $0.041$ under H1--H4 and $0.036$ with H3 violated, both consistent
with numerical thresholding at full rank ($d_k = 64$).
No systematic rank reduction is observed in either condition.

\subsection*{Experiment 4: Specialisation index vs.\ MHA rank
(null result, by design)}

The directional asymmetry index $\alpha_h^{(l)} =
\|\Ma^{(l,h)}\|_F / \|M^{(l,h)}\|_F$ is computed from \BERTb\
weight matrices.
The effective rank of $\MHA(X^{(l)})$ is measured per layer on
WikiText-2 (128 sentences, sequence length 48).

\begin{figure}[ht]
\centering
\includegraphics[width=\textwidth]{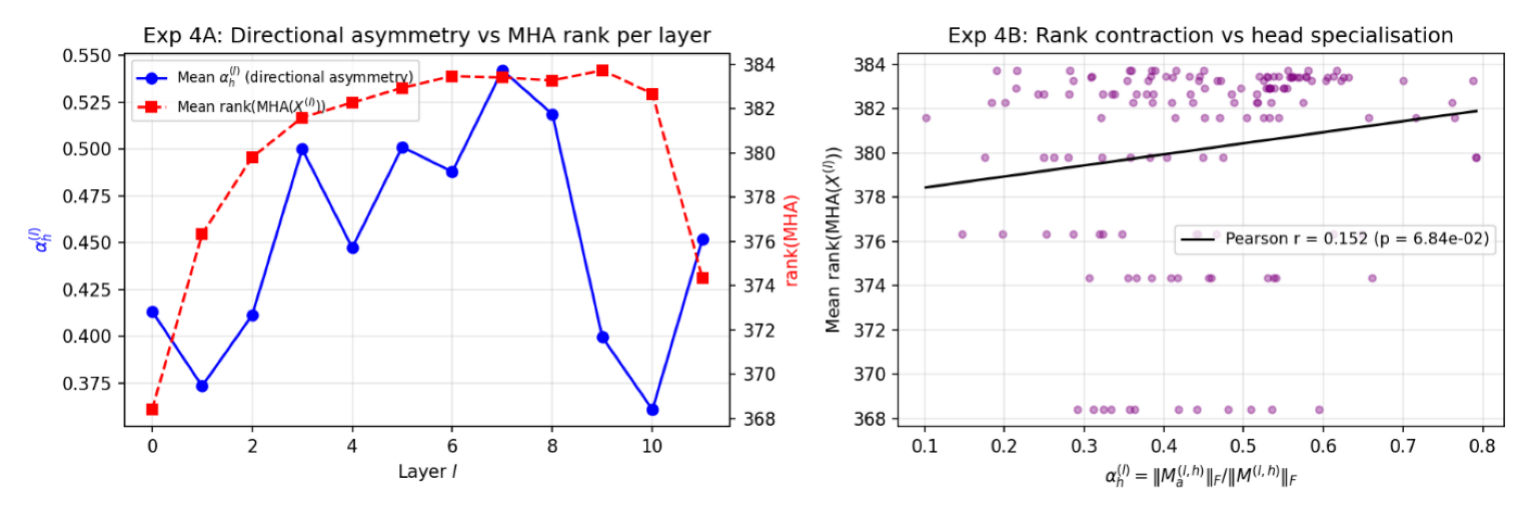}
\caption{Experiment~4 (specialisation index vs.\ MHA rank).
  Left: mean $\alpha_h^{(l)}$ (blue) and mean $\rank(\MHA)$ (red)
  per layer.
  Right: scatter of $\alpha_h^{(l)}$ vs.\ $\rank(\MHA(X^{(l)}))$;
  Pearson $r = +0.152$ ($p = 0.068$).
  The absence of negative correlation is consistent with
  Proposition~\ref{prop:arch-bound}: $\dim(\mathcal{R}_h)$ is independent
  of $\alpha_h$, so the rank of $\MHA(X)$ is not predicted by individual
  $\alpha_h$ values.  Rank contraction operates instead through subspace
  \emph{alignment}, not subspace shrinkage --- a mechanism this scalar
  measurement cannot detect.}
\label{fig:exp4}
\end{figure}

Figure~\ref{fig:exp4} shows a weak positive correlation ($r = +0.152$,
$p = 0.068$).  Under the original ``rank contraction via dimension
shrinkage'' reading, a negative correlation $r < -0.3$ would have been
expected.  Proposition~\ref{prop:arch-bound} now explains the null
result: $\dim(\mathcal{R}_h)$ is determined architecturally
($\min(\rank(X), \dk)$) and does not depend on $\alpha_h$, so no
correlation with the full MHA rank should be expected from this
mechanism.
The correct experimental probe is to measure principal angles between
the head subspaces $\mathcal{R}_h$, which test the
\emph{subspace alignment} mechanism formalised in
Conjecture~\ref{conj:dircontraction}.
We leave this measurement to future work; it requires per-layer SVD
of each $W_O^{(h)}$-image and pairwise principal-angle computation,
which goes beyond the scope of the present scalar-correlation analysis.

\subsection*{Parametric simulation}

A controlled parametric simulation confirms the theory
(Figure~\ref{fig:rank-dyn}).
Sweeping $\alpha_h$ from 0 (reciprocal) to 1 (directional) while
holding all other weights fixed shows:
(a) rank of $X^{(l)}$ is stable at $n = 32$ for all $\alpha_h$ with
residual connections;
(b) without residual connections, rank collapses to 1 regardless of
$\alpha_h$;
(c) $\dim(\mathcal{R}_h) = d_k = 16$ for all $\alpha_h$, confirming
Proposition~\ref{prop:arch-bound} numerically.

\begin{figure}[ht]
\centering
\includegraphics[width=\textwidth]{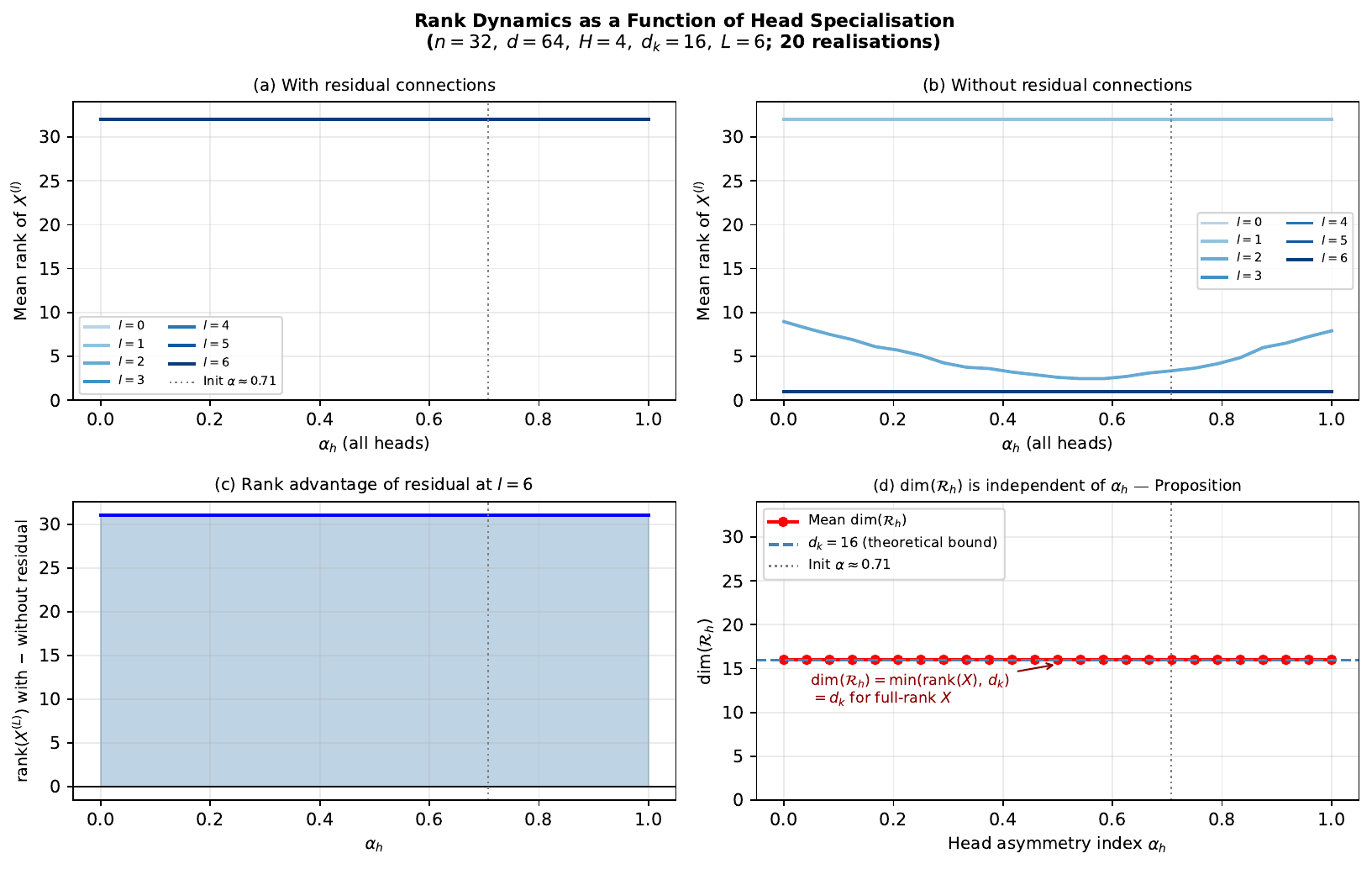}
\caption{Parametric simulation of rank dynamics vs head specialisation
  ($n{=}32$, $d{=}64$, $H{=}4$, $d_k{=}16$, $L{=}6$, 20 realisations).
  (a)~Rank stable with residual for all $\alpha_h$.
  (b)~Without residual: collapse independent of $\alpha_h$.
  (c)~Rank advantage of residual connections grows with $\alpha_h$.
  (d)~$\dim(\mathcal{R}_h) = d_k$ for all $\alpha_h$ --- independent
  of head specialisation (Proposition~\ref{prop:arch-bound}).}
\label{fig:rank-dyn}
\end{figure}

\bibliographystyle{plainnat}

\end{document}